\documentclass[letterpaper, 10 pt, journal, twoside]{IEEEtran}

\usepackage{preamble}

\title{Graph-based SLAM-Aware Exploration with Prior Topo-Metric Information
}

\author{Ruofei Bai$^{1, 2}$, Hongliang Guo$^2$, \emph{Member}, \emph{IEEE}, Wei-Yun Yau$^2$, \emph{Senior Member}, \emph{IEEE}, Lihua Xie$^1$, \emph{Fellow}, \emph{IEEE}
\thanks{$^1$Ruofei Bai and Lihua Xie are with the School of
Electrical and Electronic Engineering, Nanyang Technological University, Singapore 639798
(e-mail: \href{mailto:ruofei001@ntu.edu.sg}{ruofei001@ntu.edu.sg}; \href{mailto:elhxie@ntu.edu.sg}{elhxie@ntu.edu.sg})
}%
\thanks{$^2$Ruofei Bai, Hongliang Guo and Wei-Yun Yau are with the Institute for
Infocomm Research (I2R), Agency for Science, Technology and Research
(A*STAR), Singapore 138632
(e-mail: \href{mailto:stubair@i2r.a-star.edu.sg}{stubair@i2r.a-star.edu.sg}; \href{mailto:guo_hongliang@i2r.a-star.edu.sg}{guo\_hongliang@i2r.a-star.edu.sg}; \href{mailto:wyyau@i2r.a-star.edu.sg}{wyyau@i2r.a-star.edu.sg})}%
}

\begin{document}

\maketitle

\begin{abstract}

Autonomous exploration requires a robot to explore an unknown environment while constructing an accurate map using Simultaneous Localization and Mapping (SLAM) techniques.
Without prior information, the exploration performance is usually conservative due to the limited planning horizon. 
This paper exploits prior information about the environment, represented as a topo-metric graph, to benefit both the exploration efficiency and the pose graph reliability in SLAM.
Based on the relationship between pose graph reliability and graph topology, we formulate a SLAM-aware path planning problem over the prior graph, which finds a fast exploration path enhanced with the globally informative loop-closing actions to stabilize the SLAM pose graph.
A greedy algorithm is proposed to solve the problem, where theoretical thresholds are derived to significantly prune non-optimal loop-closing actions, without affecting the potential informative ones.
Furthermore, we incorporate the proposed planner into a hierarchical exploration framework, with flexible features including path replanning, and online prior graph update that adds additional information to the prior graph.
Simulation and real-world experiments indicate
that the proposed method can reliably achieve
higher mapping accuracy than compared
methods when exploring environments with rich topologies, while maintaining comparable exploration efficiency.
Our method has been open-sourced on GitHub.

\end{abstract}





\section{Introduction}

Autonomous exploration has been deemed as one of the important tasks for robotics, with applications to surveillance and inspection, search and rescue, etc.
Typically, a robot uses onboard sensors such as LiDAR or cameras to perceive an unknown environment and constructs an accurate map for downstream tasks, known as Simultaneous Localization and Mapping (SLAM).
In pose graph-based SLAM methods, the map can be constructed by registering onboard sensor readings at each pose into a global coordinate frame~\cite{grisetti_tutorial_2010, 9714804}. 
To get an accurate map of the environment, the robot should maintain accurate and reliable pose estimation during the exploration.

\begin{figure}
  \centering
  \includegraphics[width=0.4\linewidth]{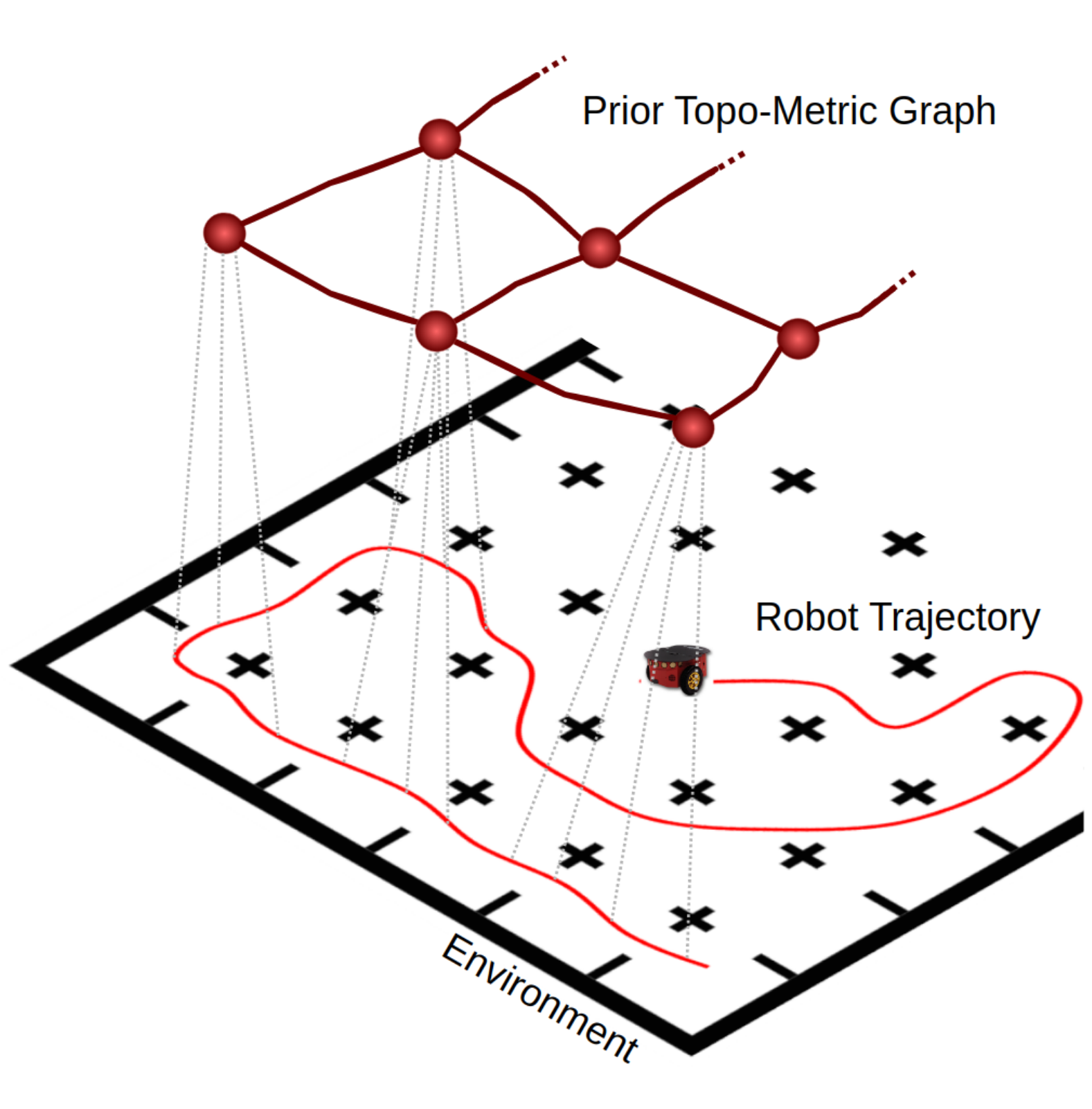}\quad
  \includegraphics[width=.55\linewidth]{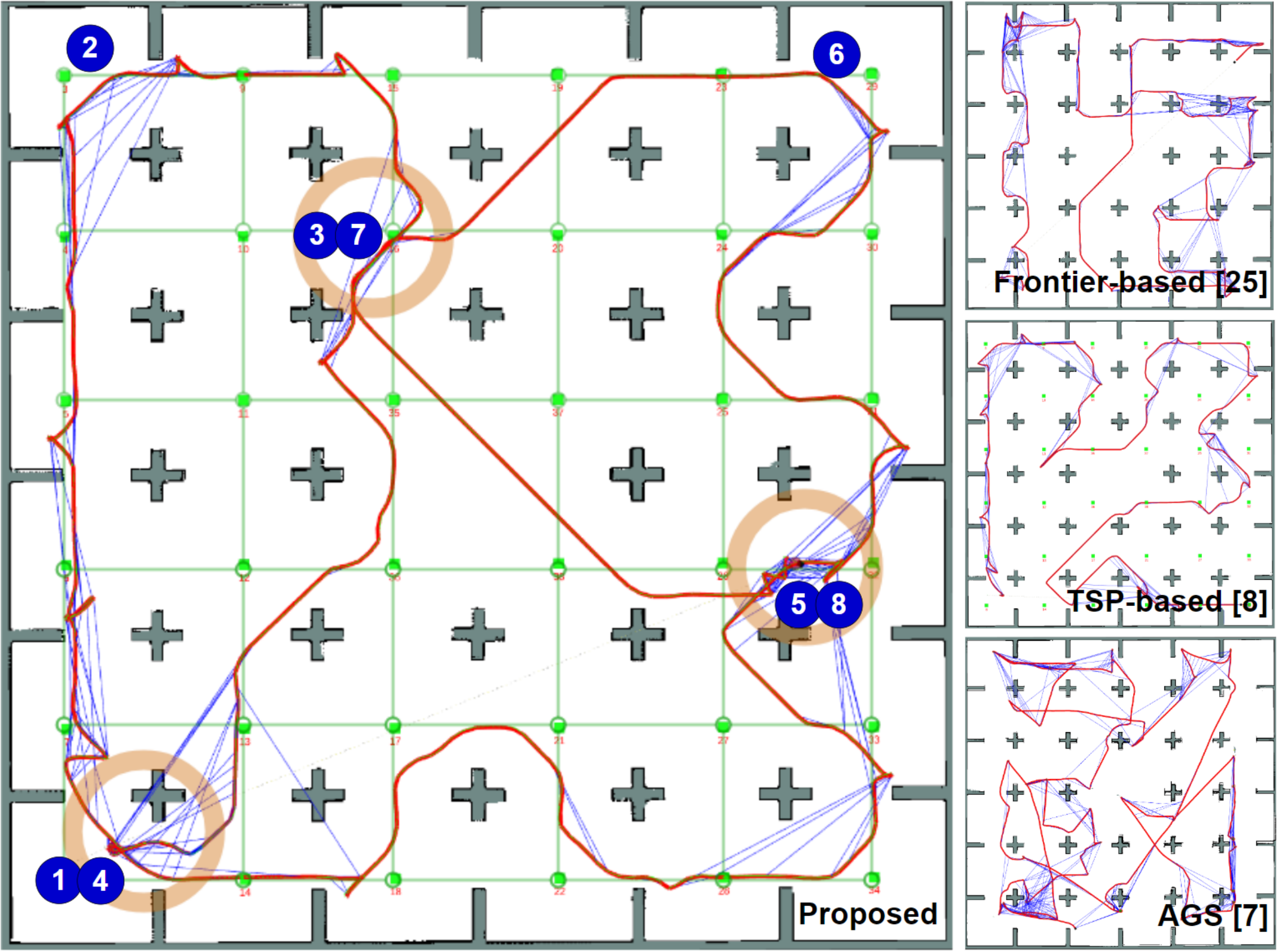}
  \caption{SLAM-Aware exploration with prior topo-metric graph. The right figures show the occupancy grid map built after exploration, the prior topo-metric graph (green), the SLAM pose graph with odometry edges (red) and loop closures (blue). 
  The numbers in blue circles indicate the order in which the robot visits different regions.
  Compared with other methods, the exploratory trajectory of the proposed method forms three big loops that quickly cover the environment while forming a globally reliable pose graph.}
  \label{fig_map_results}
  \vspace{-15pt}
\end{figure}

However, the pose estimation may accumulate drift as the robot explores the environment due to noise or environment degeneracy.
Active SLAM methods~\cite{Placed_survey_2023} are used to drive the robot to actively revisit some places to establish relative measurements, called \emph{loop closures}, to mitigate the cumulative drift in SLAM.
Exploration actions are usually selected based on their information gain, considering both the environment coverage and the robot pose uncertainty~\cite{stachniss_information_2005, Xu_CRMI_2021, carrillo_autonomous_2018, placed_fast_2021}.
However, due to the lack of prior information about the environment, 
they balance the two objectives only within a local planning horizon, thus being conservative in both exploration efficiency and finding informative loop closures.
In fact, the locally selected loop closures contribute less to mitigating the drift since the odometry drift increases with distance.

In this paper, we use the prior topological and metric information of the environment to improve the exploration performance of a robot. 
Such information can be obtained from hand-drawn maps, outdated floorplans, or previously built maps~\cite{oswald_speeding-up_2016, xue_active_2022}, and represented as a topo-metric graph to describe the positions of regions and their connectivity in the environment. 
Based on the relationship between the pose graph reliability and the graph topology~\cite{khosoussi_reliable_2019, chen_cramerrao_2021, placed_general_2023}, we formulate a SLAM-aware path planning problem over the prior graph, which finds a high-level path to guide the robot to quickly cover the environment while forming a well-connected pose graph for reliable SLAM estimation.
With the prior topo-metric graph, both exploration efficiency and pose graph reliability are effectively balanced from a global perspective compared to reactive methods, as illustrated in Fig.~\ref{fig_map_results}.

The contributions of this paper are summarized as follows: 
1) A SLAM-aware path planner over a prior topo-metric graph that explicitly considers both exploration efficiency and pose graph reliability. We derive theoretical thresholds that significantly prune non-optimal loop-closing actions to speed up an iterative greedy algorithm; 
2) A flexible hierarchical exploration framework that integrates the SLAM-aware path planner, with several features including path replanning, and online prior graph update that adds additional information to the prior graph.
Moreover, simulation and real-world experiments verify the effectiveness of the proposed method in various environments, where the proposed method reliably maintains higher mapping accuracy and comparable exploration efficiency than the compared methods.

\section{Related Works}
This section reviews the related works in autonomous exploration exploiting the prior information and recent works in relating pose graph uncertainty with the graph topology.

\subsection{Autonomous Exploration based on Prior Information}
In autonomous exploration, the robot usually applies active SLAM methods to decide its exploratory actions to balance the SLAM accuracy and the exploration efficiency.
A detailed review of active SLAM can be found in~\cite{Placed_survey_2023}.

In some applications, prior information about the environment is available, such as the floor plan, a topological sketch, etc.
Several works utilizing prior information have been proposed to improve exploration performance.
Luperto \emph{et al.}~\cite{luperto_robot_2020} use a floor plan of the environment to help more accurately evaluate the information gain of each frontier in the exploration.
Oßwald \emph{et al.}~\cite{oswald_speeding-up_2016} use a prior topo-metric graph of the environment to speed up the exploration, where a traveling salesman problem (TSP) solver finds a path over the prior topo-metric graph to guide the exploration behavior of the robot.
However, the pose uncertainty and map accuracy are not considered.
A similar situation is considered in~\cite{soragna_active_2019}, where a robot is required to traverse all edges of a prior topo-metric graph of the environment.
The distance of a robot's future path is scaled by the robot's pose uncertainty, which motivates the robot to actively close loops to reduce its pose uncertainty.
Xue \emph{et al.}~\cite{xue_active_2022} further consider the active SLAM over a prior topo-metric graph with the starting position unknown.
However, they do not consider the pose uncertainty during the exploration.
Different from existing works, this paper considers both the exploration efficiency and the pose graph reliability in a SLAM-aware path planning formulation, making full use of the prior topo-metric information of the environment.

\subsection{Graph-based SLAM Uncertainty Evaluation}
In graph-based SLAM methods~\cite{grisetti_tutorial_2010}, it is important to maintain a reliable pose graph, which reflects the accuracy of the built map in the SLAM process~\cite{Deshpande_lighthouse_2023}.
To efficiently quantify the reliability of the SLAM pose graph, 
an increasing number of works including~\cite{khosoussi_novel_2014, khosoussi_reliable_2019, chen_cramerrao_2021} and~\cite{placed_general_2023, placed_fast_2021} use the graph topology metric to evaluate the estimation uncertainty in the pose graph optimization.
Specifically, Khosoussi \emph{et al.}~\cite{khosoussi_reliable_2019} rearrange the FIM into the translational part and the rotational part and prove that each part is individually related to the pose graph Laplacian matrix weighted by the corresponding covariance matrices. 
Chen \emph{et al.}~\cite{chen_cramerrao_2021} further extend this work into 3D case.
On the other hand, 
Placed \emph{et al.}~\cite{placed_general_2023} directly encapsulate the covariance matrix of each observation into the edge weight of the Laplacian matrix to more efficiently evaluate the pose graph reliability.

The above topology-based metrics have been applied to evaluate potential loop closures in robot exploration~\cite{chen_online_2019, placed_fast_2021}.
However, the metric is either evaluated in a reactive manner, or limited in frontier evaluation.
Instead, we use the above relationship to find a distance-efficient path with globally informative loop closures over the prior topo-metric graph, which guides the robot to form a well-connected pose graph for reliable SLAM estimation.

\section{Preliminaries}
\subsection{Graph Laplacian}

Given a connected graph $\mathcal{G}$ with $n+1$ vertices and $m$ edges, the graph Laplacian matrix is defined as:
\begin{equation}
    \Lp^{\circ} = \mathbf{B}^{\circ}\mathbf{B}^{\circ\top} = \sum\nolimits_{j = 1}^{m} \mathbf{B}_j \mathbf{B}_j^{\top}\in \mathbb{R}^{(n+1)\times (n+1)},
\end{equation}
where $\mathbf{B}^{\circ}\in \mathbb{R}^{(n+1) \times m}$ is the incidence matrix of $\mathcal{G}$, $\mathbf{B}_j$ is the $j$-th column vector of $\mathbf{B}^{\circ}$ that only has non-zero values at two indices corresponding to the vertices connected by the $j$-th edge.
Each row of $\mathbf{B}^{\circ}$ corresponds to a vertex and each column corresponds to an edge in $\mathcal{G}$.
The weighted Laplacian matrix of the graph $\mathcal{G}$ is defined as $\Lp^{\circ}_{\gamma} = 
    \sum_{j = 1}^{m} \gamma_{j} \mathbf{B}_{j}\mathbf{B}_{j}^{\top}$, 
where $\gamma_j$ is the weight of $j$-th edge in $\mathcal{G}$.

If one vertex is anchored in $\mathcal{G}$, i.e., the corresponding row is removed from the incidence matrix $\mathbf{B}^{\circ}$, the obtained Laplacian matrix is called the reduced Laplacian matrix, denoted as $\Lp \in \mathbb{R}^{n \times n}$.
The weighted reduced Laplacian matrix $\Lp_{\gamma}$ can also be defined similarly.

\subsection{Relating Pose Graph Reliability with Graph Laplacian}
\label{sec_relationship}

In graph-based SLAM algorithms~\cite{grisetti_tutorial_2010}, a pose graph is incrementally built, denoted as $\mathcal{G}^{\text{pose}} = \langle \mathcal{X}, \mathcal{Z} \rangle$, where each node in $\mathcal{X}$ corresponds to a robot pose, and an edge $\zij \in \mathcal{Z}$ represents the relative motion between two poses $x_i$ and $x_j$, established by either odometry measurements or loop closure detection.
The pose graph optimization aims to find the maximum likelihood estimation (MLE) of all poses $\mathcal{X}$, given all relative observations $\mathcal{Z}$.
The objective function in pose graph optimization is defined as $\min_{\mathcal{X}} \mathbf{F}(\mathcal{X}) = \sum_{\zij\in \mathcal{Z}} \mathbf{F}_{ij}(\mathcal{X}) = \sum_{\zij\in \mathcal{Z}} \eij^{\top}\mathbf{\Sigma}^{-1}_{ij}\boldsymbol{e}_{ij}$, 
where $\boldsymbol{e}_{ij}= \zij - \hat{z}_{ij}(x_i, x_j)$ is the residual error between the predicted relative observation $\hat{z}_{ij}$ and the actual observation $\zij$; $\mathbf{\Sigma}_{ij}$ is the covariance matrix of $\zij$.
The poses in $\mathcal{X}$ can be optimized using nonlinear optimization methods like the Gauss-Newton method.
The Hessian matrix $\mathbf{H}$ in pose graph optimization is derived as
\begin{equation}
\small
\raisebox{0.5ex}{$\mathbf{H} 
= \frac{1}{2} \sum\nolimits_{\zij \in \mathcal{Z}} \mathbf{J}_{ij} ^{\top} \boldsymbol{\Sigma}^{-1}_{ij} \mathbf{J}_{ij}
= \frac{1}{2} \sum\nolimits_{\zij \in \mathcal{Z}} \mathbf{B}_{ij}\mathbf{B}_{ij}^{\top} \otimes \widetilde{\boldsymbol{\Sigma}}_{ij}^{-1}$},
\label{eq_hessian}
\end{equation}
where $\mathbf{J}_{ij}$ is the Jacobian matrix of the error function $\mathbf{F}_{ij}$ w.r.t. $\mathcal{X}$,  $\widetilde{\boldsymbol{\Sigma}}_{ij}$ is the covariance matrix of the observation $\zij$ with coordinate transformation, and $\mathbf{B}_{ij}\mathbf{B}_{ij}^{\top}$ is the Laplacian factor of the edge $\zij$ in the pose graph.
As $\mathbf{F}_{ij}$ only relates to $x_i$ and $x_j$, the corresponding $\mathbf{J}_{ij}$ only has non-zero values at indices $i$ and $j$, which shares a similar structure as the column vector $\mathbf{B}_{ij}$ of the incidence matrix of the pose graph.
The details of the derivation are referred to~\cite{placed_general_2023}.

The Hessian matrix is also known as the \emph{observed} FIM~\cite{khosoussi_novel_2014}, denoted as $\mathbb{I}(\mathcal{X})$.
In practice, the observed FIM $\mathbb{I}(\mathcal{X})$ computed at the MLE of $\mathcal{X}$ is quantified to reflect the uncertainty of $\mathcal{X}$, using scaler functions of $\mathbb{I}(\mathcal{X})$, among which the D-optimal is shown to be superior in capturing the global uncertainty in all dimensions of the poses~\cite{carrillo_comparison_2012}.
The D-optimal is a scalar function from Theory of Optimal Experimental Design~\cite{pazman1986foundations} and is defined as $\dopt{\Lp} = \det(\Lp)^{\frac{1}{n}}$ for an $n\times n$ matrix.
According to~\cite{placed_general_2023}, 
the D-optimal of $\mathbb{I}(\mathcal{X})$ can be approximated by:
\begin{equation}
\small
\begin{aligned}
&\dopt{\mathbb{I}(\mathcal{X})} 
=
\dopt{\sum\nolimits_{\zij\in\mathcal{Z}} \mathbf{B}_{ij} \mathbf{B}_{ij}^{\top} \otimes \widetilde{\boldsymbol{\Sigma}}_{ij}^{-1}} \\
\approx &
\dopt{\sum\nolimits_{\zij\in\mathcal{Z}} \mathbf{B}_{ij} \mathbf{B}_{ij}^{\top} \cdot \dopt{\widetilde{\boldsymbol{\Sigma}}_{ij}^{-1}} } 
= 
\dopt{\Lp_{\gamma}}.
\end{aligned}
\label{eq_relationship}
\end{equation}
The subscript $\gamma$ in $\Lp_{\gamma}$ indicates that the reduced Laplacian matrix of the pose graph is weighted by the D-optimal of the inverse of the covariance matrices for each edge in the pose graph.
Compared with the Hessian matrix $\mathbf{H}$, the weighted Laplacian matrix $\mathbf{L}_{\gamma}$ has smaller dimensions that only depend on the number of poses, and is more efficient in evaluating the pose graph reliability.
Note that Eq.~(\ref{eq_relationship}) establishes the relationship between the pose graph reliability and graph topology, and the latter will be used in this paper for SLAM uncertainty evaluation.

\section{SLAM-Aware Path Planning Over Prior Graph}
\label{sec_problem_formulation}

The key idea of this paper is to find a high-level path over the prior topo-metric graph of the environment, to guide the low-level exploration behavior of the robot to achieve quick coverage while forming a well-connected pose graph for reliable SLAM estimation. 
The proposed hierarchical exploration framework is shown in Fig.~\ref{fig_framework} and will be introduced in detail in Sec.~\ref{sec_hierarchical_framework}.
This section focuses on the high-level planning part.
We formulate the SLAM-aware path planning problem based on the prior graph, and propose the greedy algorithm with effective pruning mechanisms to solve it.

\begin{figure}[t]
\vspace{1pt}
\centering\includegraphics[width=\columnwidth]{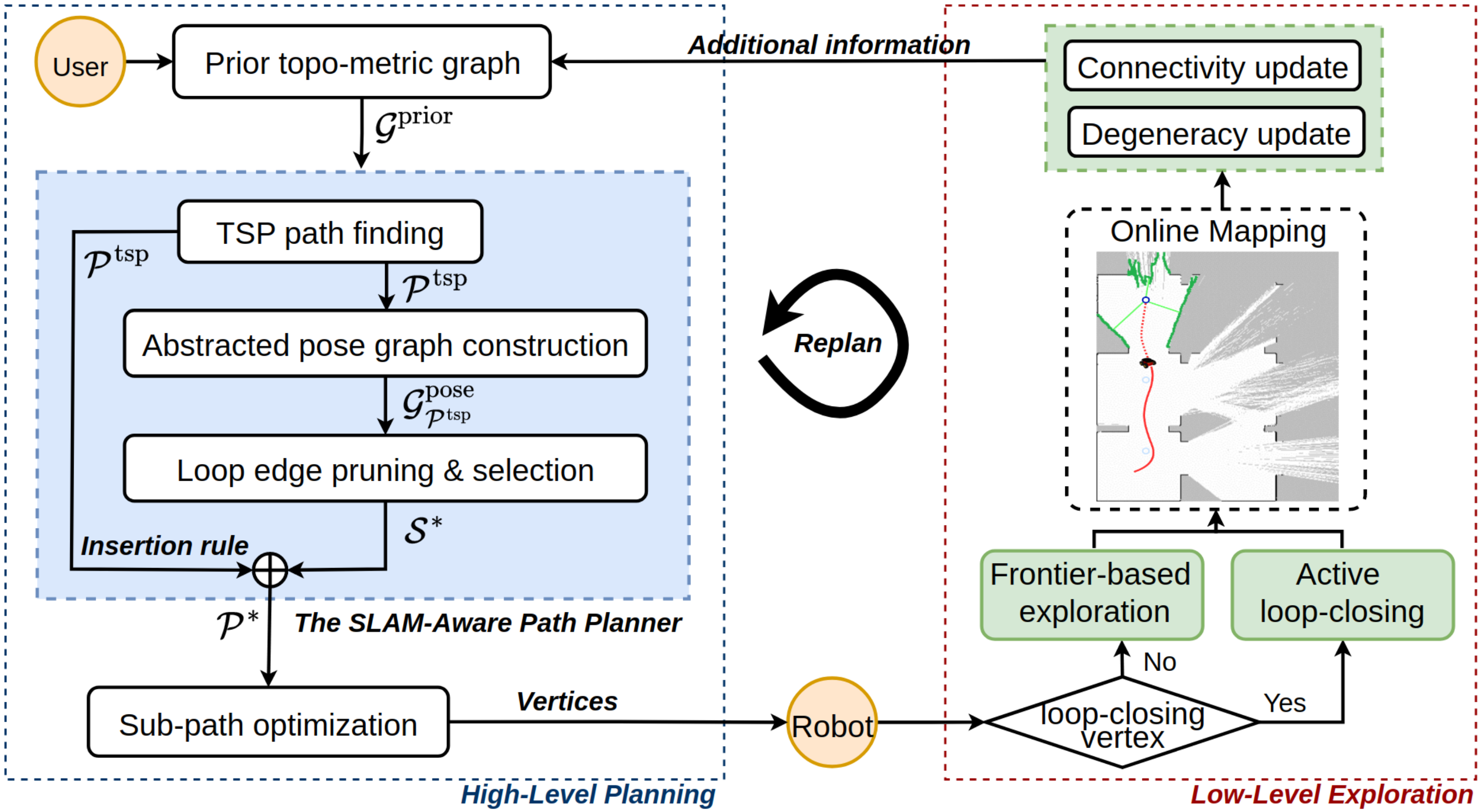}
	\caption{The diagram of the proposed hierarchical exploration framework, including high-level planning and low-level navigation. The SLAM-aware path planner takes the prior topo-metric graph $\gprior$ from users and outputs a high-level path $\mathcal{P}^{*}$. The robot follows vertices in $\mathcal{P}^{*}$ to perform either active loop-closing or frontier-based exploration within a region. The prior graph is updated online, based on which the replanning and sub-path optimization modules opportunistically find a better path for exploration.} 
 \label{fig_framework}
 \vspace{-15pt}
\end{figure}

\subsection{Representation of the Prior Topo-Metric Information}
The prior knowledge about the coarse structure of an environment can be represented by a prior topological-metric graph, defined as $\gprior = \langle \mathcal{V}, \mathcal{E}, \omega\rangle$, where $\mathcal{V}\subseteq \mathbb{R}^{2}/\mathbb{R}^{3}$ includes set of vertices distributed in the environment, where each vertex corresponds to a convex region to be explored, like room or corridor, etc.
Generally, one vertex per convex region is sufficient~\cite{oswald_speeding-up_2016}.
The set $\mathcal{E} \subseteq \mathcal{V}\times \mathcal{V}$ includes all edges, where an edge $\left (v_i, v_j\right ) \in \mathcal{E} $ indicates there exists a traversable path between $v_i $ and $v_j $.
$\omega(v_i, v_j)$ denotes the distance between $v_i$ and $v_j$ in $\gprior$.
We assume the prior graph is connected, and the starting pose of the robot w.r.t. $\gprior$ is given initially. 
The starting vertex of the robot is denoted as $v_{\text{start}}$.

\subsection{Two-Stage Strategy for SLAM-Aware Path Planning}
\label{sec_offline_formulation}

Given the prior topo-metric graph $\gprior$, the problem is to find a walk $\mathcal{P}$ over $\gprior$ that covers all vertices in $\gprior$, while forming a well-connected pose graph topology for reliable SLAM estimation.
The problem is similar to solving a traveling salesman problem (TSP), which finds the shortest path to visit each vertex in a graph once and return to the starting vertex, and it is NP-hard to solve~\cite{JUNGER1995225}.
While the TSP only considers the distance cost, we need to further consider the reliability of the pose graph induced by the path.
To reduce the computational complexity, we propose the following two-stage strategy to solve the problem:
\begin{enumerate}
    \item[(1)] Find an exploration path starting from $v_{\text{start}}$ to quickly cover all vertices in $\gprior$, where existing TSP solvers can be used;
    \item[(2)] Select and insert informative and distance-efficient loop-closing actions along the exploration path to improve the pose graph reliability.
\end{enumerate}

Specifically, to construct a TSP instance in Stage (1), the prior graph $\gprior$ is first transformed into a complete graph, where the inter-vertex distance is defined as the shortest path length between the two vertices in $\gprior$.
Existing TSP solvers can then be used to find the shortest path over the complete graph starting from $v_{\text{start}}$ to visit all vertices in $\gprior$ exactly once.
Note the original TSP definition has additional returning constraints, which can be eliminated by resetting the distance of all incoming edges of $v_{\text{start}}$ as zero and deleting the final vertex in the return TSP path, which is also known as Open-loop TSP~\cite{chieng2014performance}.
After obtaining the TSP path, we connect consecutive vertices in the path with the shortest path between them in $\gprior$, to recover the actual path that follows the connectivity constraints of $\gprior$.
The finally obtained path, which is actually a walk over $\gprior$, is denoted as $\mathcal{P}^{\text{tsp}} = ( v_{i_{1}}, ...,  v_{i_{|\mathcal{P}^{\text{tsp}}|}} )$,
where $\forall~k\in [1:|\mathcal{P}^{\text{tsp}}|-1]$, $v_{i_{k}}\in \mathcal{V}$, $(v_{i_{k}}, v_{i_{k+1}})\in \mathcal{E}$; and $v_{i_{1}} = v_{\text{start}}$.
The walk $\mathcal{P}^{\text{tsp}}$ visits all vertices in $\mathcal{V}$ and ensures complete coverage of the environment.

\subsection{Loop Edge Selection over Abstracted Pose Graph}

As the obtained walk $\mathcal{P}^{\text{tsp}}$ only aims to quickly cover vertices in $\gprior$, Stage (2) further improves the resulting pose graph reliability of the path, by adding informative and distance-efficient loop-closing actions along $\mathcal{P}^{\text{tsp}}$.
However, while the exploration efficiency can be directly quantified by the distance cost of $\ptsp$, the pose graph reliability can not be evaluated directly without the actual pose graph constructed in SLAM.
On the positive side, in our proposed framework, the walk $\ptsp$ over $\gprior$ captures the abstracted structure of the robot's actual pose graph, because the robot will follow the high-level guidance of $\ptsp$ to explore the environment.
Therefore, we define the \emph{abstracted} pose graph from $\ptsp$ to provide a high-level approximation of the robot's actual pose graph, which will be used to identify informative loop-closing actions along the walk $\mathcal{P}^{\text{tsp}}$ to reducing the SLAM uncertainty.\footnote{A similar idea of high-level abstraction of the pose graph is proposed in hierarchical pose graph optimization~\cite{grisetti_hierarchical_2010}, where the poses in the robot's actual pose graph are clustered to construct an abstracted pose graph, and then only the abstracted pose graph is optimized to facilitate efficiency.}.

\begin{definition}[Abstracted pose graph]
Given a walk $\mathcal{P}$ over a graph $\gprior = \langle \mathcal{V}, \mathcal{E}, \omega \rangle$, an abstracted pose graph corresponding to $\mathcal{P}$ is defined as $\gpose_{\mathcal{P}} = \langle \mathcal{X}, \mathcal{Z} \rangle$, where $\mathcal{X}\subseteq \operatorname{SE}(2)/\operatorname{SE}(3)$ and $\mathcal{Z}\subseteq \mathcal{X}\times \mathcal{X}$.
The poses in $\mathcal{X}$ have a one-to-one correspondence with vertices in $\mathcal{V}$ covered by $\mathcal{P}$, and an edge $( x_i, x_j )\in \mathcal{Z}$ exists if the edge $(\mathcal{M}(x_i), \mathcal{M}(x_j)) \in \mathcal{E}$ and is covered by $\mathcal{P}$, where $\mathcal{M}: \mathcal{X}\xrightarrow{} \mathcal{V}$ is a mapping function from $\mathcal{X}$ to $\mathcal{V}$.
The subscripts of the poses in $\mathcal{X}$ are assigned according to the order in which their corresponding vertices are visited by $\mathcal{P}$.
\label{def_abstract_pose_graph}
\end{definition}

By following Def.~\ref{def_abstract_pose_graph}, the abstracted pose graph corresponding to $\mathcal{P}^{\text{tsp}}$ is constructed and denoted as $\gpose_{\ptsp}$.
The repeated visits of vertices and edges in $\gprior$ by $\ptsp$ are only counted once for simplicity, so $\gpose_{\ptsp}$ has the same topology as a subgraph of $\gprior$ covered by $\ptsp$.
Note that the orientation of poses is not explicitly defined in Def.~\ref{def_abstract_pose_graph}, because the exact values of pose orientations do not affect the evaluation of SLAM uncertainty, as suggested in Eq.~(\ref{eq_relationship}).

\begin{remark}
If the edge lengths in the original $\gprior$ vary significantly, we can add additional vertices along the long edges in $\gprior$ to make them balanced for better approximation of the abstracted pose graph.
Meanwhile, the connectivity information in $\gprior$ remains unchanged.
\end{remark}

With the abstracted pose graph $\gpose_{\mathcal{P}^{\text{tsp}}}$, we then identify candidate loop-closing actions, called \emph{loop edges}, along the walk $\ptsp$, and select the most informative and distance-efficient ones to be inserted into the walk to improve the pose graph reliability.
The candidate loop edges in $\gpose_{\mathcal{P}^{\text{tsp}}}$ are defined as follows:

\begin{definition}[Loop edge]
A loop edge connects two poses that are not directly connected in $\gpose_{\mathcal{P}^{\text{tsp}}}$. Given $\gpose_{\mathcal{P}^{\text{tsp}}} = \langle \mathcal{X}, \mathcal{Z} \rangle$, the set of all candidate loop edges is defined as $\mathcal{S} = \{z_{ij} = ( x_i, x_j) \mid x_i, x_j \in \mathcal{X}; i > j; \zij \notin \mathcal{Z}\}$.
\label{def_loop_edge}
\end{definition}

A loop edge is not a loop closure, but a continuous movement to connect two poses in the abstracted pose graph, to improve the pose graph reliability.
If a loop edge $(x_i, x_j)$ is selected, the vertex $\mathcal{M}(x_i)$ in $\mathcal{P}^{\text{tsp}}$ will be replaced by $\left(\mathcal{M}(x_i), \mathcal{M}(x_j), \mathcal{M}(x_i)\right)$, and we refer this operation as \emph{insertion rule}.
It means that after the robot covers $\mathcal{M}(x_i)$, it will move from $\mathcal{M}(x_i)$ to $\mathcal{M}(x_j)$ to establish loop closures and then return to $\mathcal{M}(x_i)$ to continue exploration.
The vertex $\mathcal{M}(x_j)$ is called a loop-closing vertex.
With such, the extra distance cost introduced by the loop edge $(x_i, x_j)$ can be conveniently counted as $2\cdot \omega(\zij)$, here we abuse the notation $\omega(\cdot)$ and let $\omega(\zij)$ denotes the shortest distance between $\mathcal{M}(x_i)$ and $\mathcal{M}(x_j)$ in $\gprior$.
Note that in actual exploration, the robot will skip previously visited vertices as introduced in Sec.~\ref{sec_hierarchical}.

The loop edge selection problem is formulated as follows:

\begin{problem}
Given a prior topo-metric graph $\gprior$, an initial walk $\mathcal{P}^{\text{tsp}}=(v_{i_{1}}, ..., v_{i_{|\mathcal{P}^{\text{tsp}}|}})$ over $\gprior$, and a set $\mathcal{S}$ of candidate loop edges, find a subset $\mathcal{S}' \subseteq \mathcal{S}$ that maximize the following objective function:
\begin{equation*}
\small
     \mathcal{J}(\mathcal{S}') 
    = \frac{D\text{-}opt\left (\Lp_{\gamma}(\mathcal{P}^{\text{tsp}}) + \sum_{\zij \in \mathcal{S}'} \gamma_{ij} \cdot \mathbf{B}_{ij} \mathbf{B}_{ij}^{\top}  \right)}{\mathcal{D}(\mathcal{P}^{\text{tsp}}) + 2\cdot \sum_{\zij \in \mathcal{S}'} \omega(\zij)},
\end{equation*}
where $\Lp_{\gamma}(\mathcal{P}^{\text{tsp}})$ is the reduced Laplacian matrix of the abstracted pose graph $\gpose_{\mathcal{P}^{\text{tsp}}}$;
$\mathbf{B}_{ij} \mathbf{B}_{ij}^{\top}$ is the Laplacian factor corresponding to the loop edge $\zij$ in $\gpose_{\ptsp}$; $\gamma_{ij}$ is the weight of the edge $\zij$\footnote{In initial prior graph, all edges have the same covariance matrices, and will be updated online during exploration, as introduced in Sec.~\ref{sec_hierarchical}. The initial value is set as $\operatorname{Diag}\{0.1m, 0.1m, 0.001 rad\}$ in the 2D experiments.}; 
$\mathcal{D}(\ptsp) = \sum_{k = 1}^{|\ptsp|-1} \omega(v_{i_{k}}, v_{i_{k+1}})$ is the distance cost of $\ptsp$; and $2\cdot \sum_{\zij \in \mathcal{S}'} \omega(\zij)$ is the extra distance cost introduced by the selected loop edges in $\mathcal{S}'$.
\label{problem_2}
\end{problem}

The Problem~\ref{problem_2} aims to find loop edges that balance the exploration distance and the pose graph reliability, with the numerator of the objective function $\mathcal{J}(\cdot)$ being the topology metric that quantifies the estimation uncertainty in SLAM, when adding several loop edges into the abstracted pose graph.
An intuitive interpretation of $\mathcal{J}(\cdot)$ is that the extra distance traveled by the robot must contribute to reducing the averaged uncertainty of the pose graph, otherwise the loop-closing actions will not be performed.
Here the TSP path $\mathcal{P}^{\text{tsp}}$ provides a natural lower bound for the objective function $\mathcal{J}(\cdot)$ for further improvement.
The algorithm to solve Problem~\ref{problem_2} will be introduced in the next subsection.

\subsection{Greedy Solution with Pruning Thresholds}
\label{sec_greedy_filtering}
This section presents a greedy-based iterative algorithm to solve Problem~\ref{problem_2}. 
We derive the pruning thresholds to speed up the greedy algorithm, which prune non-optimal candidate loop edges beforehand without affecting the informative ones. 
The proposed greedy algorithm with the pruning mechanism is summarized in Alg.~\ref{alg_greedy}.

Specifically, in each iteration, the greedy algorithm selects the loop edge that has the maximum contribution to the objective function $\mathcal{J}$, adds it into $\mathcal{S}'$, and then re-computes the contribution of all remaining candidate edges based on the updated $\mathcal{S}'$.
The algorithm terminates until no candidate loop edge can further improve the objective value.
The greedy algorithm is time-consuming as it needs to re-evaluate the contribution of all candidate edges in each iteration.
Assume we already have a set $\mathcal{S}'$ of selected loop edges, and let $\mathcal{P}'$ denote the obtained path by inserting the selected loop edges into $\mathcal{P}^{\text{tsp}}$ following the insertion rule.
Now consider adding another loop edge $\zij$ into $\mathcal{S}'$, the objective will be:
\begin{equation*}
\small
\begin{aligned}
    &\mathcal{J}(\mathcal{S}' \cup \{\zij\}) 
    = 
    \frac{\dopt{\Lp_{\gamma}(\mathcal{P}') + \gamma_{ij} \cdot \mathbf{B}_{ij} \mathbf{B}_{ij}^{\top}}}{\mathcal{D}(\mathcal{P}') + 2\omega(\zij)} \\
    &=
    \mathcal{J}(\mathcal{S}') \cdot 
    \frac{\left(1 + \gamma_{ij} \cdot \mathbf{B}_{ij}^{\top} \Lp_{\gamma}^{-1}(\mathcal{P}') \mathbf{B}_{ij} \right)^{\frac{1}{n}}}{\left( 1 + \frac{2\omega(\zij)}{\mathcal{D}(\mathcal{P}')}\right)} 
= 
    \mathcal{J}(\mathcal{S}') \cdot \delta(\zij, \mathcal{P}').
\end{aligned}
\label{eq_edge_contribution}
\end{equation*}
Here the incremental term $\delta(\zij, \mathcal{P}')$ quantifies the contribution of the loop edge $\zij$ to the objective $\mathcal{J}$.
It depends on the current $\mathcal{P}'$ that changes in each iteration.
As introduced before, the set $\mathcal{S}$ includes edges with the order of $\mathcal{O}(|\mathcal{V}|^2)$.
If the greedy algorithm terminates in $k$ iterations, we have to evaluate edges of order $\mathcal{O}(k\cdot |\mathcal{V}|^2)$ in total.
To reduce the time complexity, we establish the following proposition to prune non-optimal loop edges from $\mathcal{S}$.

\begin{proposition}
    Under the assumption $\mathcal{D}(\mathcal{P}') \le 2\cdot \mathcal{D}(\mathcal{P}^{\text{tsp}})$\footnote{The assumption $\mathcal{D}(\mathcal{P}') \le 2\cdot \mathcal{D}(\mathcal{P}^{\text{tsp}})$ is reasonable because in the worst case, the robot can go back to the starting position by following exactly the initial path $\pp^{\text{tsp}}$.The algorithm should find a better path than the worst case. The assumption is also verified to hold in our experiments.}, if a loop edge $\zij\in \mathcal{S}$ satisfies
\begin{equation}
\small
    \frac{\left(1 + \gamma_{ij} \cdot \mathbf{B}_{ij}^{\top} \Lp^{-1}_{\gamma}(\mathcal{P}^{\text{tsp}}) \mathbf{B}_{ij} \right)^{\frac{1}{n}}}{1 + \frac{\omega(\zij)}{\mathcal{D}(\mathcal{P}^{\text{tsp}})}}
    \le
    1,
\label{eq_add_edge_constraint}
\end{equation}
then for $\mathcal{S}' \subseteq \mathcal{S}$, it holds that
$
\mathcal{J}(\mathcal{S}' \cup \{\zij\})
\le 
\mathcal{J}(\mathcal{S}').
$
\label{proposition_1}
\end{proposition}

\begin{proof}
Since the reduced Laplacian matrix $\Lp_{\gamma}$ is positive definite, its inverse $\Lp_{\gamma}^{-1}$ is also positive definite.
According to Lemma~$9$ of~\cite{khosoussi_reliable_2019}, 
for two positive definite matrix $\Lp_{\gamma}(\mathcal{P}')$ and $\Lp_{\gamma}(\mathcal{P}^{\text{tsp}})$, 
$\Lp_{\gamma}(\mathcal{P}') \succeq \Lp_{\gamma}(\mathcal{P}^{\text{tsp}})$ iff 
$\Lp_{\gamma}^{-1}(\mathcal{P}^{\text{tsp}}) \succeq \Lp_{\gamma}^{-1}(\mathcal{P}')$.
Since $\mathcal{P}^{\text{tsp}}$ is a subgraph of $\mathcal{P}'$, 
we have $\Lp_{\gamma}(\mathcal{P}') \succeq \Lp_{\gamma}(\mathcal{P}^{\text{tsp}})$ (Lemma~$6$ of~\cite{khosoussi_reliable_2019}), and thus $\Lp_{\gamma}^{-1}(\mathcal{P}^{\text{tsp}}) \succeq \Lp_{\gamma}^{-1}(\mathcal{P}')$.
It is obvious that $\gamma_{ij} \cdot \mathbf{B}_{ij}^{\top} \Lp^{-1}_{\gamma}(\mathcal{P}^{\text{tsp}}) \mathbf{B}_{ij} \ge \gamma_{ij} \cdot \mathbf{B}_{ij}^{\top} \Lp^{-1}_{\gamma}(\mathcal{P}') \mathbf{B}_{ij}$.
Therefore,  
\begin{equation*}
\small
\left(1 + \gamma_{ij} \mathbf{B}_{ij}^{\top} \Lp^{-1}_{\gamma}(\mathcal{P}^{\text{tsp}}) \mathbf{B}_{ij} \right)^{\frac{1}{n}} 
    \ge
\left(1 + \gamma_{ij} \mathbf{B}_{ij}^{\top} \Lp^{-1}_{\gamma}(\mathcal{P}')\mathbf{B}_{ij} \right)^{\frac{1}{n}}.
\label{eq_submodular}
\end{equation*}
Considering $\mathcal{D}(\mathcal{P}') \le 2\cdot \mathcal{D}(\mathcal{P}^{\text{tsp}})$, 
it holds that $\small
    1 + \frac{2\cdot \omega(\zij)}{\mathcal{D}(\mathcal{P}')} 
    \ge 
    1 + \frac{\omega(\zij)}{\mathcal{D}(\mathcal{P}^{\text{tsp}})}$.
Combining the above two inequalities and Eq.~(\ref{eq_add_edge_constraint}), it holds that $\delta(\zij, \mathcal{P}') < 1$, thus proves the Prop.~\ref{proposition_1}.
\end{proof}

With Prop.~\ref{proposition_1}, we can further establish a distance threshold $\omega_{\text{max}}$ to prune long loop edges from $\mathcal{S}$.
We first find the loop edge $\zij^{*} = \arg \max_{\zij} \left(1 + \gamma_{ij} \cdot \mathbf{B}_{ij}^{\top} \Lp^{-1}_{\gamma}(\mathcal{P}^{\text{tsp}}) \mathbf{B}_{ij} \right)^{\frac{1}{n}}$ with computational complexity $\mathcal{O}(\lvert \mathcal{V} \rvert^{2})$. 
And then the distance threshold $\omega_{\text{max}}$ satisfies
$\small \frac{\left(1 + \gamma_{ij}^{*} \cdot \mathbf{B}_{ij}^{*\top} \cdot \Lp^{-1}_{\gamma}(\mathcal{P}^{\text{tsp}}) \cdot \mathbf{B}_{ij}^{*} \right)^{\frac{1}{n}}}{1 + \frac{\omega_{\text{max}}}{\mathcal{D}(\mathcal{P}^{\text{tsp}})}}
=
1$ according to Prop.~\ref{proposition_1}.
If a loop edge $\zij\in \mathcal{S}$ has distance $\omega(\zij) > \omega_{\text{max}}$, it will always satisfy Eq.~(\ref{eq_add_edge_constraint}) and thus can be removed from $\mathcal{S}$ directly.
Furthermore, we can also use $\mathcal{P}'$ rather than $\mathcal{P}^{\text{tsp}}$ in Eq.~\ref{eq_add_edge_constraint} to update the pruning threshold (Line 10, Alg.~\ref{alg_greedy}).
The correctness can be proved similarly as in Prop.~\ref{proposition_1}.
The updated threshold is shown to be tighter and thus keeps pruning edges in greedy iterations, as in Fig.~\ref{fig_ratio}(c).

By applying Alg.~\ref{alg_greedy}, we can obtain a suboptimal set $\mathcal{S}^{*}$ of selected loop edges.
And then the final exploration path, denoted as $\mathcal{P}^{*}$, is constructed by inserting the loop edges in $\mathcal{S}^{*}$ into their corresponding positions in $\mathcal{P}^{\text{tsp}}$ following the insertion rule.
The obtained path $\mathcal{P}^{*}$ will provide high-level guidance to the robot to explore the entire environment in a hierarchical framework, which will be introduced in Sec.~\ref{sec_hierarchical_framework}.

\SetAlgoSkip{myskip}
\begin{algorithm}[t]
\small
\SetKwInOut{Input}{Input}\SetKwInOut{Output}{Output}
\SetKwInOut{Return}{Return}
\caption{GreedyLoopEdgeSelection}
\label{alg_greedy}
\Input{$\gprior$, $\mathcal{P}^{\text{tsp}}$, $\mathcal{S}$}
Find $\zij^{*}$ in $\mathcal{S}$, and calculate $\omega_{\text{max}}$.\\
$\forall \zij\in \mathcal{S}$, if $\omega(\zij) \ge \omega_{\text{max}}$, remove $\zij$ from $\mathcal{S}$ .\\
$\mathcal{S}' \xleftarrow{} \{\}$, $iter \xleftarrow{} 0$, $\mathcal{P}' \xleftarrow{} \mathcal{P}^{\text{tsp}}$.\\
\While{$true$}{
    \For{each $\zij \in \mathcal{S}$}{
        Compute $\delta(\zij, \mathcal{P}')$.\\
        Update Eq.~(\ref{eq_add_edge_constraint}) using $\mathcal{P}'$ rather than $\mathcal{P}^{\text{tsp}}$.\\
        \If{$\zij$ satisfies Eq.~(\ref{eq_add_edge_constraint})}{
            Remove $\zij$ from $\mathcal{S}$.\\
        }
    }
    Find $\zij^{\text{max}} \xleftarrow{} \arg \max_{\zij \in \mathcal{S}} \delta(\zij, \mathcal{P}')$.\\
    \eIf{$\delta(\zij^{\text{max}}, \mathcal{P}') > 1$}{
        Add $\zij^{\text{max}}$ into $\mathcal{S}'$, remove $\zij^{\text{max}}$ from $\mathcal{S}$.\\
        Update $\mathcal{P}' \xleftarrow{} \mathcal{P}^{\text{tsp}} \cup \mathcal{S}'$ following insertion rule.\\
        $iter \xleftarrow{} iter + 1$.\\
    }{
        Break.\\
    }
}
\KwRet{$\mathcal{S}^{*} \xleftarrow{} \mathcal{S}'$.}
\end{algorithm}

\section{Hierarchical Exploration and Replanning}
\label{sec_hierarchical_framework}

This section presents a hierarchical exploration framework incorporating the proposed SLAM-aware path planner to achieve SLAM-aware exploration, as shown in Fig.~\ref{fig_framework}.


\subsection{Hierarchical Autonomous Exploration}
\label{sec_hierarchical}

As shown in Fig.~\ref{fig_framework}, in our framework, the SLAM-aware path planner takes the prior topo-metric graph $\gprior$ from users as input, and outputs the path $\mathcal{P}^{*}$ as high-level guidance to the robot in exploration.
The detailed steps are as follows:
\begin{enumerate}
    \item The robot navigates to the first region (vertex) provided by the path $\mathcal{P}^{*}$;
    \item Active loop-closing: if current vertex is a loop-closing vertex, the robot will follow a segment of its previous trajectory to establish loop closures in current region.
    \item Frontier exploration: otherwise, the robot exhaustively explores frontiers within current region (corresponds to a vertex in $\gprior$);
    \item The robot navigates to the next region (vertex) in $\mathcal{P}^{*}$.
\end{enumerate}

The exploration strategy is robust to slight variations in the exact positions of vertices specified in $\gprior$, because the robot is considered to have reached a vertex $v_i$ once it enters the vicinity of $v_i$, without the need to exactly reach that vertex. 
Moreover, the robot can skip a non-loop closing vertex if the corresponding region has no frontiers.
The resulting exploration trajectory is therefore mainly motivated by exploring the unknown regions, rather than exactly following the path $\mathcal{P}^{*}$, as shown in Fig.~\ref{fig_map_results} and Fig.~\ref{fig_prior_graphs}.

\subsection{Online Prior Graph Update}

\begin{figure}[t]
\vspace{1pt}
\centering\includegraphics[width=\columnwidth]{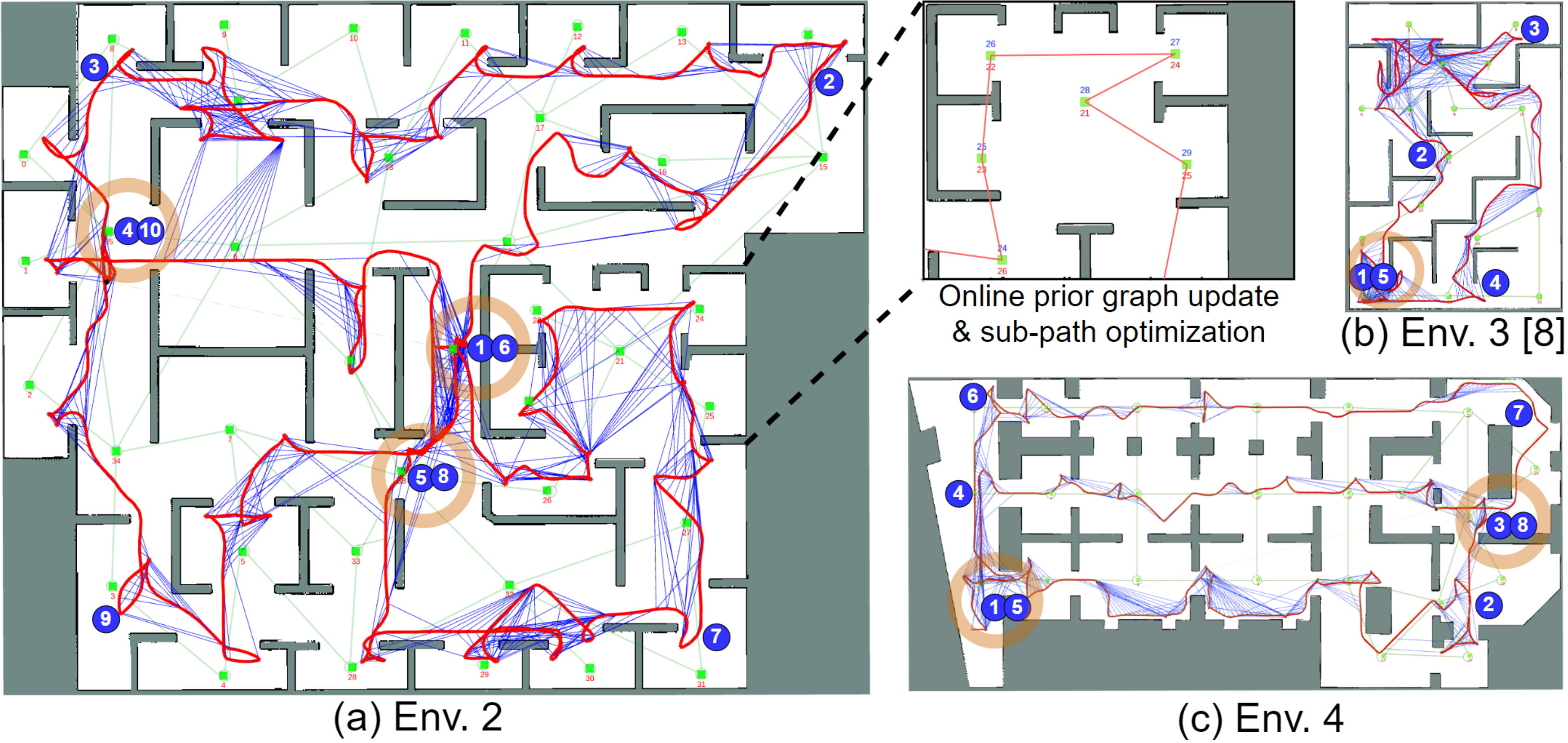}
\captionsetup{skip=0pt} 
\caption{The occupancy maps built by the proposed exploration method in various environments. 
The sizes of environment maps are: (a) $87m\times 69m$, (b) $40m \times 57m$, (c) $138m \times 66m$.
The Env.~$1$ ($74m\times 74m$) is shown in Fig.~\ref{fig_map_results}.
Each image also shows the robot's pose graph (red), loop closures (blue), the prior topo-metric graph (vertices and edges in green), and the ID of each vertex (numbers in red).
The numbers in blue circles indicate the order in which the robot visits different regions, and the loop-closing vertices are marked by orange rings.
Fig.~\ref{fig_prior_graphs}(a) also shows a zoom-in area of Env.~$2$.
The initial prior graph of Env.~$2$ only provides a star-like connectivity between the vertices.
After the online connectivity update, the sub-path optimization module finds a better sub-path for exploration.}
\label{fig_prior_graphs}
\vspace{-10pt}
\end{figure}

This section introduces the online update of the prior graph as the robot obtains more information about the environment during exploration, to benefit the path planning process.

\textbf{Degeneracy update:} 
The environmental degeneracy can be reflected by the covariance of the relative motion estimation between two poses, e.g., the covariance in the scan-matching process~\cite{zhang_degeneracy_2016, Konolige_Karto_2010}.
The robot is more likely to establish reliable loop closures in feature-rich regions that have low covariance in relative motion estimation.
After a robot covers a region (vertex) $v_i$, the degeneracy matrix of this region is calculated by the averaged covariance matrix 
of the closest $N$ edges in the robot's pose graph to vertex $v_i$ (we set $N = 5$ in our experiments).
For regions that have not been visited yet, their degeneracy matrices are defined as the averaged covariance matrix of all edges in the robot's pose graph.
We then update the covariance matrix of an edge $z_{ij}$ in the abstracted pose graph as the averaged degeneracy matrices of the two regions it connects.

\textbf{Connectivity update:} To supplement the potentially limited connectivity information in the initial prior graph, an independent thread will subscribe the SLAM-built map and run the A* algorithm to incrementally establish edges not specified in the prior graph.

\subsection{Path Replanning and Sub-Path Optimization}
\label{sec_replanning}

After performing a loop-closing action, the framework will replan the robot's path by calling the SLAM-aware path planner based on the updated prior graph $\gprior$.
The robot will follow the better path of the existing path and the replanned path in subsequent exploration.
Moreover, once the prior graph is updated, the sub-path of $\mathcal{P}^{*}$ starting from the current vertex (where the robot is located) until the next loop-closing vertex is also optimized by solving a local TSP problem, trying to find a shorter sub-path based on the updated prior graph while not affecting the loop-closing behavior of the robot.
An example is shown in Fig.~\ref{fig_prior_graphs}(a).

\section{Experiments}

We implement the proposed algorithm as an exploration plugin to the \emph{nav2d}\footnote{\url{https://github.com/skasperski/navigation_2d}} ROS package.
We use \emph{pyconcorde}\footnote{\url{https://github.com/jvkersch/pyconcorde}} as the TSP solver.
The robot trajectory error is evaluated using \emph{evo}\footnote{\url{https://github.com/MichaelGrupp/evo}}. 
All experiments are conducted on a desktop with an i9-13900 CPU and 32 GB of RAM.
The experimental environments Env.~$1$-$4$ and the associated prior topo-metric graphs are shown in Fig.~\ref{fig_map_results} and Fig.~\ref{fig_prior_graphs}.
We compare our method with the frontier-based method~\cite{Yamauchi_frontier_1997},
TSP-based method~\cite{oswald_speeding-up_2016} (not open-source and we implement our version), and the active graph-based SLAM (AGS) method in~\cite{placed_fast_2021}. 
Note the proposed method and the TSP-based method require prior topo-metric information, while others do not have such constraints.
All these implementations are open-sourced.\footnote{\url{https://github.com/bairuofei/Graph-Based_SLAM-Aware_Exploration}}


\subsection{Efficiency of Pruning Threshold in Loop Edge Selection}
\label{sec_filtering_experiment}

This section evaluates the effectiveness of the pruning thresholds in Alg.~$1$, based on randomly generated prior graphs with different areas, i.e., $10m\times 10 m$, $15m\times 15 m$, $20m\times 20 m$, $30m\times 30 m$, as shown in Fig.~\ref{fig_ratio}(a).
The graphs are derived from grid-like structures, with $5\%$ of vertices and edges randomly removed to introduce diverse topology.
We also add the Gaussian noise with zero mean and $\sigma = 0.2 m$ to the positions of vertices in prior graphs.
The edges in the prior graphs are assumed to have constant measurement covariance $\operatorname{diag}\{0.1m, 0.1m, 0.001rad\}$. 

Tab.~\ref{tab_filter_first_iteration} records the number of candidate edges before and after pruning in the first iteration of Alg.~\ref{alg_greedy}, where the distance threshold $\omega_{\text{max}}$ prunes around half of the candidate edges in most cases, and Prop.~\ref{proposition_1} further prunes over $90\%$ of the candidate edges in $\mathcal{S}$. 
The pruning ratio gets smaller as the size of the graph increases, indicating the proposed thresholds work better in large-scale environments or more fine-grained prior topo-metric graphs.
Moreover, the runtime efficiency of Alg.~$1$ with and without the proposed pruning thresholds is also shown in Tab.~\ref{tab_filter_first_iteration}, where $t^{\text{Alg.1}}_{\text{prune}}$ is almost an order of magnitude faster than $t^{\text{Alg.1}}_{\text{no-prune}}$.
Fig.~\ref{fig_ratio}(b) shows the proportion of remaining loop edges during the iteration of Alg.~\ref{alg_greedy} using the updated pruning threshold Eq.~(\ref{eq_add_edge_constraint}) in Prop.~\ref{proposition_1}, as introduced in Sec.~\ref{sec_greedy_filtering}.
In all cases, the updated threshold in Prop.~\ref{proposition_1} prunes a significant number of candidate loop edges in the first several iterations, indicating the updated threshold gets even tighter as Alg.~\ref{alg_greedy} iterates.

In addition, we also record the time spent to solve the TSP problem over the example prior graphs in Tab.~\ref{tab_filter_first_iteration}. 
Two methods are compared, a branch-and-bound (BnB) method using \emph{pyconcorde} and a greedy-based method\footnote{\url{https://github.com/dmishin/tsp-solver}}.
Although the BnB method usually outputs better TSP paths, it becomes inefficient when dealing with hundreds of vertices, in which case the greedy method can be used to get a valid sub-optimal TSP path within a second.
In the exploration of Env.~$1$-$4$, we use the \emph{pyconcorde} package which can output the TSP path within $0.1$ seconds.

\begin{table}[t]
    \setlength{\abovecaptionskip}{0cm} 
    \setlength{\belowcaptionskip}{-0.3cm}
\scriptsize
  \caption{Effectiveness of Pruning Thresholds in Alg.~\ref{alg_greedy}}
  \label{tab_filter}
  \centering
  \begin{tabular}{l|p{1.3cm}|p{1.3cm}|p{1.3cm}|p{1.3cm}}
    \hline
        Size ($m^2$) &  $10\times 10$ & $15\times 15$ & $20\times 20$ & $30\times 30$ \\
    \hline
        $|\mathcal{V}|$ & $95$ & $220$ & $395$ & $895$ \\
        $\lvert\mathcal{S}\rvert$ & $4,371$ & $22,578$ & $71,630$ & $364,214$\\
        Aft. $\omega_{\text{max}}$  & $4,058$ & $13,295$ & $29,008$ & $69,864$\\
        Aft. Prop.~\ref{proposition_1} & $329$  & $1,272$  & $2,190$ & $4,743$ \\
        Ratio & $7.53\%$  & $5.63\%$  & $3.06\%$ & $1.63\%$\\
        \hline
        $t_{\text{BnB}}^{\text{tsp}}$ (sec) &  $0.41$ & $1.28$ & $42.15$&  -\\
        $t_{\text{greedy}}^{\text{tsp}}$ (sec) &  $0.004$&$0.02$&  $0.08$ & $0.44$\\
        $t_{\text{no-prune}}^{\text{Alg.1}}$(sec) &  $0.48$ &  $5.53$ & $50.21$ & $1181.61$ \\
        $t_{\text{prune}}^{\text{Alg.1}}$ (sec) &  $0.07$ &$0.59$ & $3.01$& $21.47$\\
    \hline
  \end{tabular}
  \raggedright
  Note: "-" means cannot return a solution after $30$ minutes; $t^{\text{Alg.1}}_{\text{no-prune}}$ and $t^{\text{Alg.1}}_{\text{prune}}$, the runtime of Alg.~\ref{alg_greedy} w/wo pruning thresholds; $t_{\text{BnB}}^{\text{tsp}}$ and $t_{\text{greedy}}^{\text{tsp}}$, the TSP solving time using a branch-and-bound method and a greedy method.
\label{tab_filter_first_iteration}
\end{table}

\begin{figure}[t]
\centering
\vspace{-10pt}
\subfloat[]{\includegraphics[width=.3\linewidth]{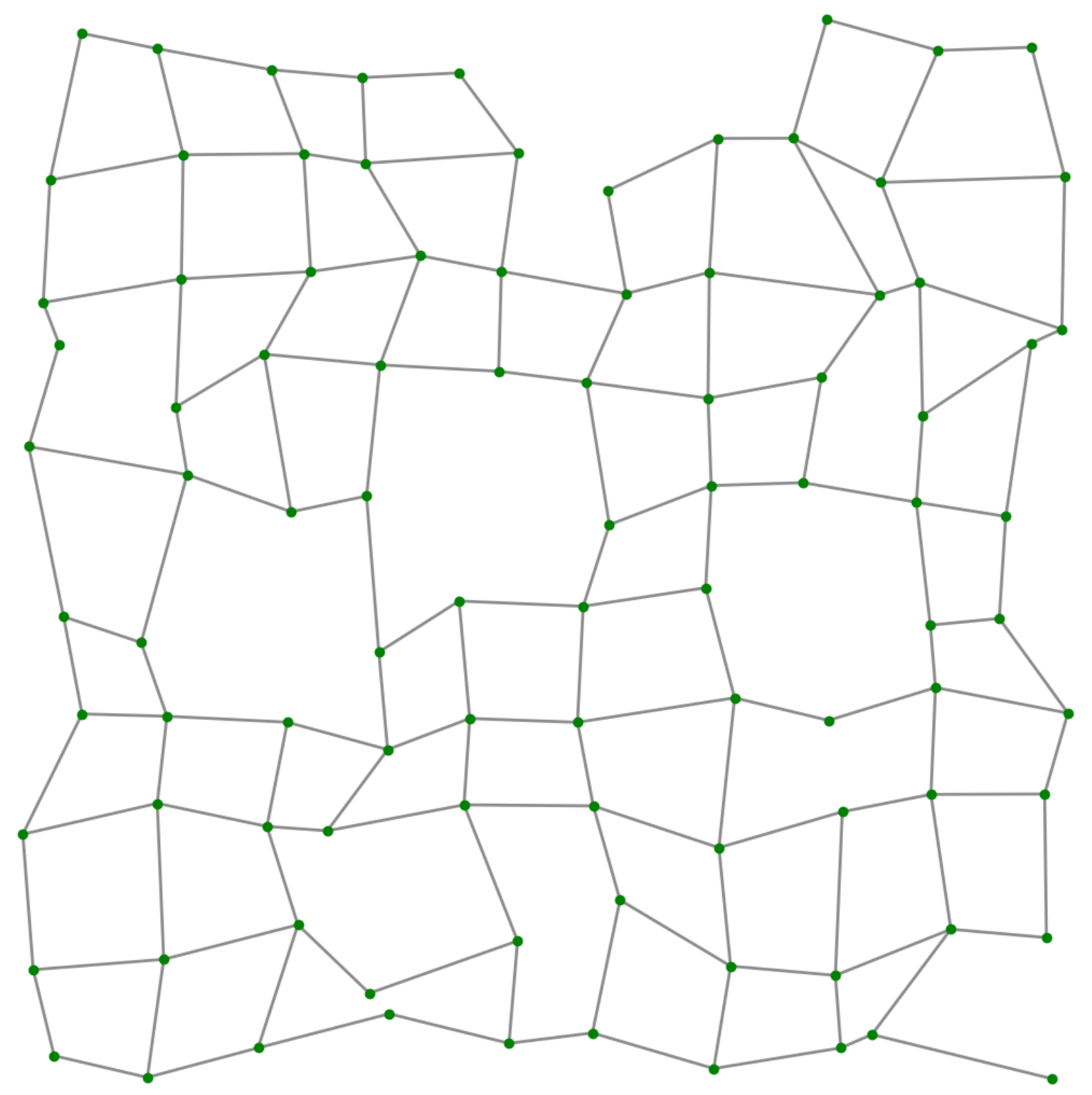}}\hspace{2pt}
\subfloat[]{\includegraphics[width=.3\linewidth]{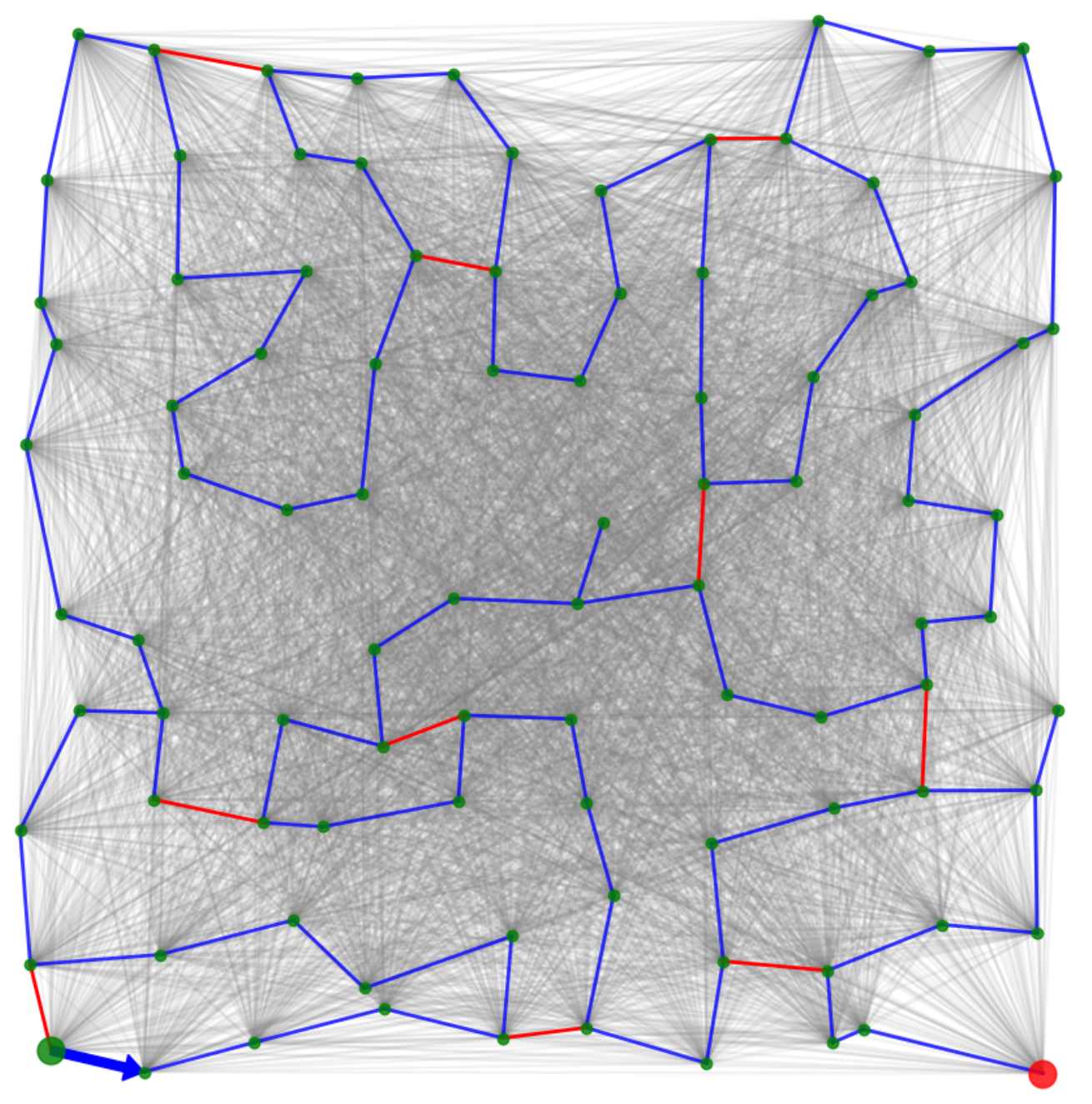}}\hspace{2pt}
\subfloat[]{\includegraphics[width=.34\linewidth]{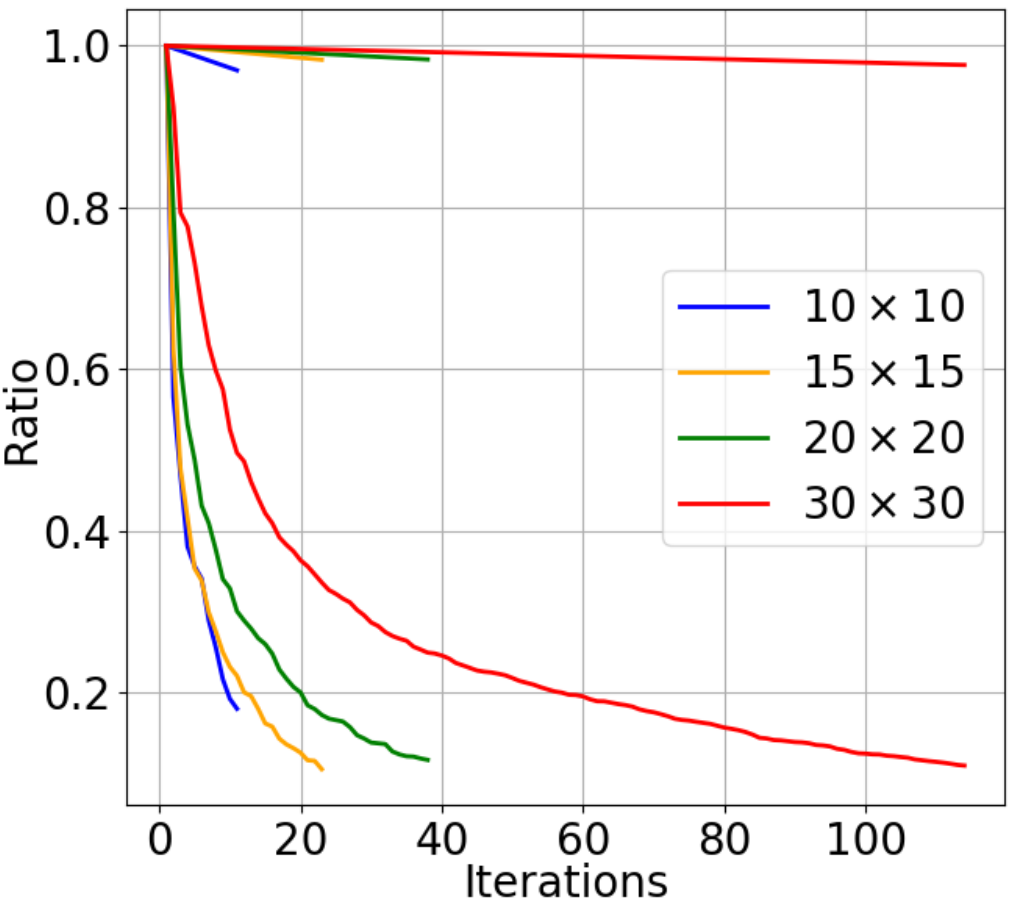}}
\caption{(a) An example topo-metric graph derived from a $10m\times 10m$ grid-like structure; (b) the initial TSP walk $\mathcal{P}^{\text{tsp}}$ (blue), the candidate (gray) and the selected (red) loop edges. (c) The ratio of the remaining candidate loop edges after the first iteration of Alg.~\ref{alg_greedy} until termination. The curves in the upper area of (c) show the corresponding cases without edge pruning.}
\label{fig_ratio}
\vspace{-15pt}
\end{figure}

\subsection{Evaluation of Exploration Performance}

This section evaluates the exploration performance of our method in Env.~1-4, as shown in Fig.~\ref{fig_map_results} and Fig.~\ref{fig_prior_graphs}.
All methods run the same SLAM method called \emph{Karto}~\cite{Konolige_Karto_2010} with both odometry estimation and loop closure detection.

Tab.~\ref{tab_exploration} records the averaged exploration results over five independent runs.
In all cases, our method has the smallest absolute pose error, and a comparable exploration efficiency (quantitated by the total distance $d_{\text{total}}$) with the frontier-based method 
except in Env. $1$, which is close to the free space and thus our method needs additional movement for active loop closing.
The TSP-based method covers the whole environment with the shortest distance; however, this comes at the cost of larger pose errors.
The active SLAM method~\cite{placed_fast_2021} gets better exploration efficiency in Env.~$3$ than our methods, and good mapping accuracy in Env.~$4$, because it evaluates the information gain and pose graph reliability when selecting frontiers.
However, as it is mainly motivated by exploring frontiers, the method fails to actively form informative loop closures to reduce the odometry error.
It is also verified in Tab.~\ref{tab_exploration} that a larger averaged node degree $k$ (i.e., more loop closures) does not always lead to a smaller pose error, while only the globally informative ones can effectively mitigate the accumulated odometry drift in exploration.

\begin{table}[t]
    \setlength{\abovecaptionskip}{0cm} 
    \setlength{\belowcaptionskip}{-0.3cm}
  \caption{Exploration Results in Different Environments}
  \tiny
  \centering
  \begin{tabular}{>{\centering\arraybackslash}p{0.2cm}|>{\centering\arraybackslash}p{0.8cm}|c|>{\centering\arraybackslash}p{0.4cm}|c|c|c}
    \hline
    Env. & Method &$n_{\text{pose}}$  & $k$ & APE / m & $t_{\text{total}}$/sec & $d_{\text{total}}$/m\\
    \hline
    \multirow{4}{*}{1} & Fron.~\cite{Yamauchi_frontier_1997} & $532\pm 17$ & $3.02$ & $0.28\pm 0.13$ & $932.72$ & $405.42$\\
                       & TSP~\cite{oswald_speeding-up_2016} & $512\pm 30$  & $2.91$ & $0.41\pm 0.35$  & $932.18$ & $385.02$\\
                       & AGS\cite{placed_fast_2021} & $821\pm 64$ & $3.34$ & $0.33\pm 0.09$  & $4077.06$ & $627.75$ \\
                       & Ours & $625\pm 19$ & $3.02$ & $\mathbf{0.20\pm 0.07}$  & $1262.58$ & $488.69$ \\
    \hline
    \multirow{4}{*}{2} & Frontier & $1062\pm 63$  & $3.86$ & $0.21\pm 0.14$  & $1428.56$ & $747.07$ \\
               & TSP & $922\pm 29$  & $3.46$ & $0.23\pm 0.24$  & $1287.14$ & $689.40$\\
               & AGS & $1499\pm 274$ & $5.51$ & $0.28\pm 0.01$  & $6568.32$ & $989.66$ \\
               & Ours & $1060\pm 58$ & $3.85$ & $\mathbf{0.19\pm 0.10}$  & $1418.04$ & $783.50$\\
    \hline
    \multirow{4}{*}{3} & Frontier & $458\pm 49$  & $3.86$ & $0.14\pm 0.09$  & $619.58$ & $349.60$\\
                   & TSP & $343\pm 52$  & $3.74$ & $0.14\pm 0.12$  & $486.38$ & $276.06$\\
                   & AGS & $377\pm 35$  & $3.45$ & $0.21\pm 0.02$  & $1255.16$ & $304.92$ \\
                   & Ours & $437\pm 43$  & $4.27$ & $\mathbf{0.11\pm 0.06}$  & $642.0$ & $348.67$\\
    \hline
    \multirow{4}{*}{4} & Frontier & $1093\pm 148$  & $3.82$ & $0.30\pm 0.18$  & $1669.66$ & $804.272$\\
               & TSP & $849\pm 58$  & $3.47$ & $0.78\pm 0.62$  & $1223.66$ & $637.89$\\
               & AGS & $1013\pm 67$  & $3.47$ & $0.29\pm 0.05$  & $4790.3$ & $848.31$ \\
               & Ours & $947\pm 53$ & $3.66$ & $\mathbf{0.24\pm 0.13}$  & $1710.18$ & $748.88$\\
    \hline
  \end{tabular}
\raggedright
Note: $n_{\text{pose}}$, the number of poses in pose graph; $k=(2\times n_{\text{edge}})/n_{\text{pose}}$, the averaged node degree in pose graph; APE, the RMSE of absolute pose error; $t_{\text{total}}$ and $d_{\text{total}}$, the total time and distance used to cover the environment.
\label{tab_exploration}
\end{table}

\begin{figure}[t]
\centering
\vspace{-10pt}
\subfloat[Env. 1]{\includegraphics[width=.241\linewidth]{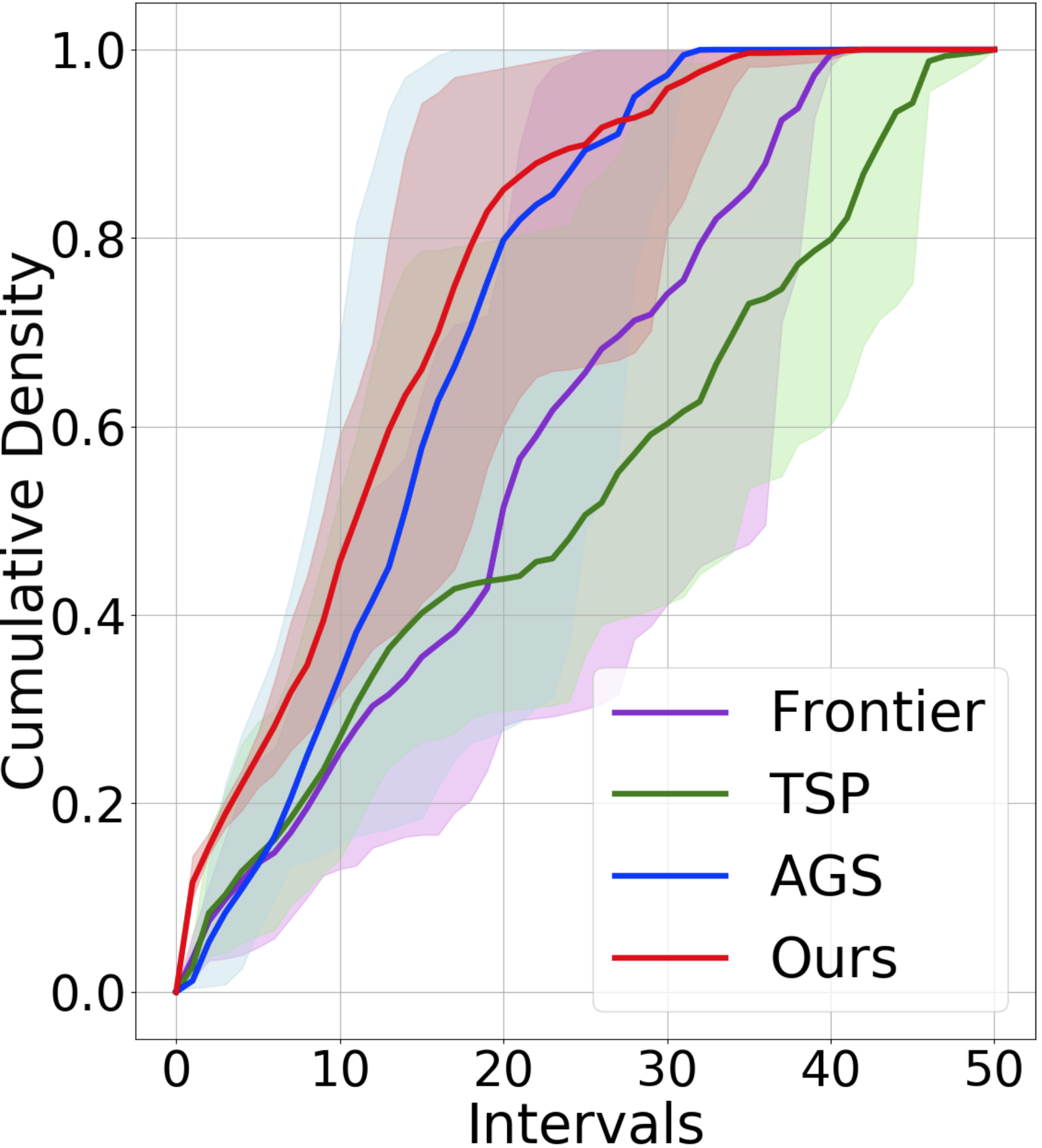}}\hspace{-0.4pt}
\subfloat[Env. 2]{\includegraphics[width=.241\linewidth]{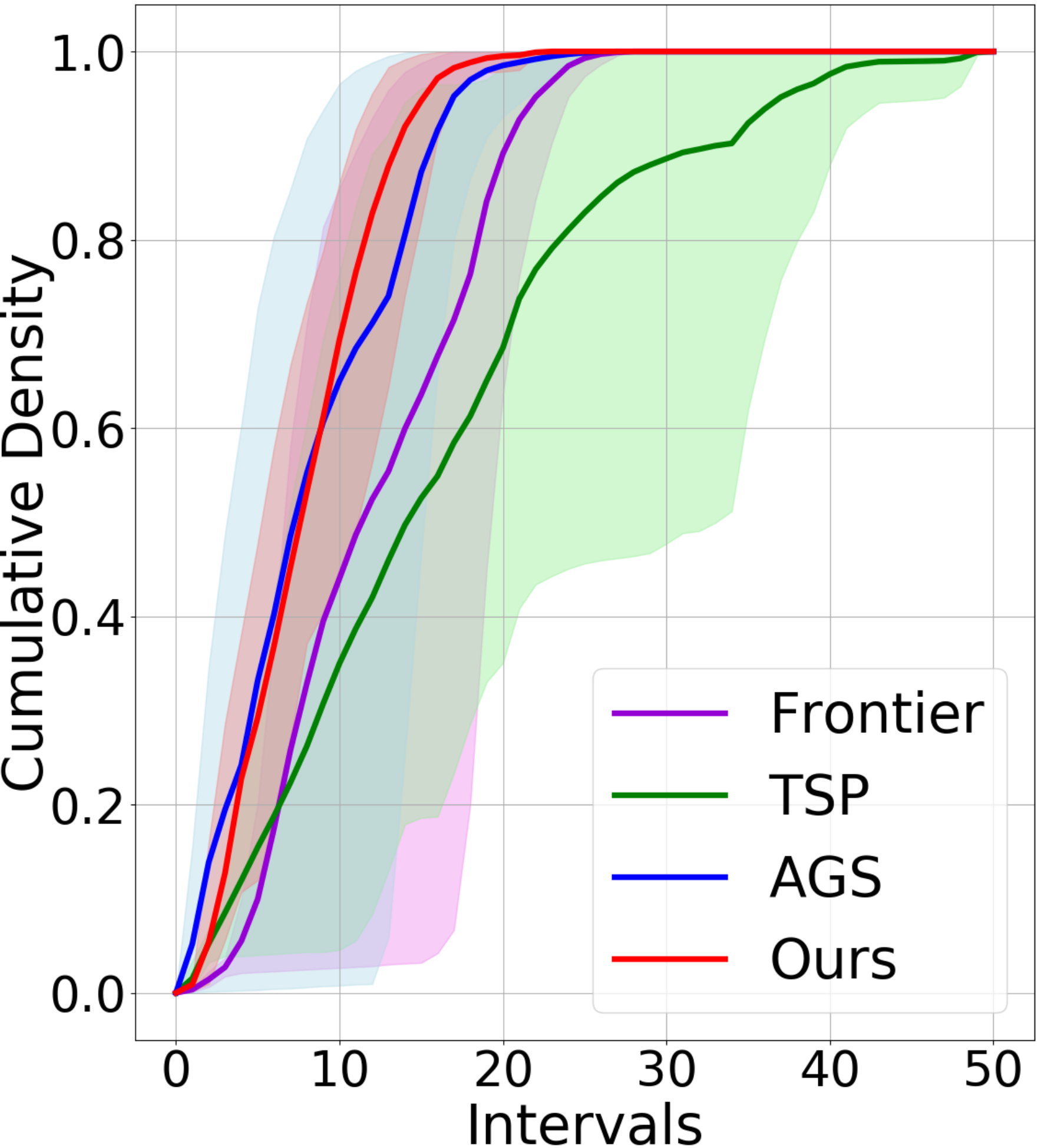}}\hspace{-0.4pt}
\subfloat[Env. 3]{\includegraphics[width=.241\linewidth]{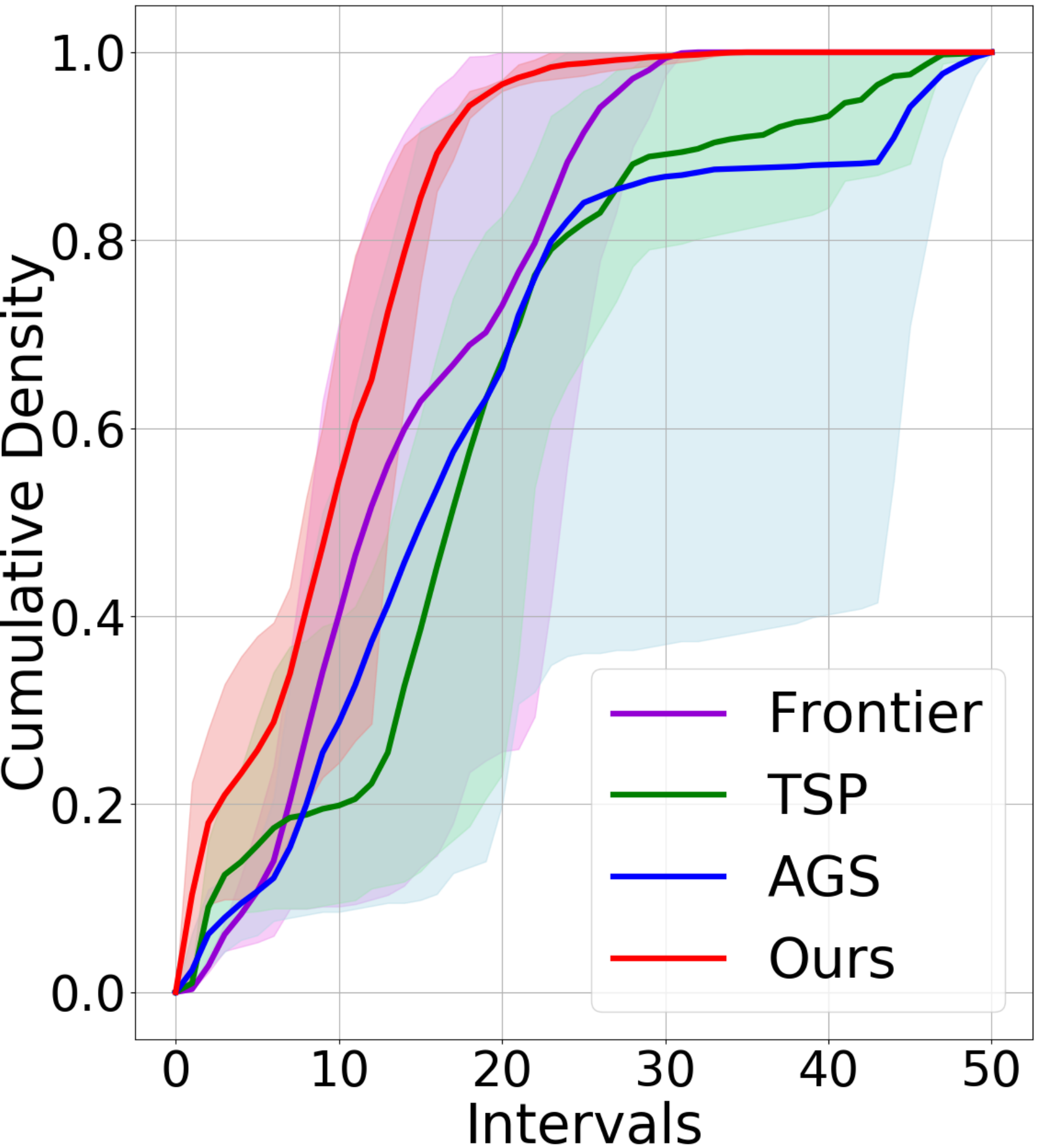}}\hspace{-0.4pt}
\subfloat[Env. 4]{\includegraphics[width=.241\linewidth]{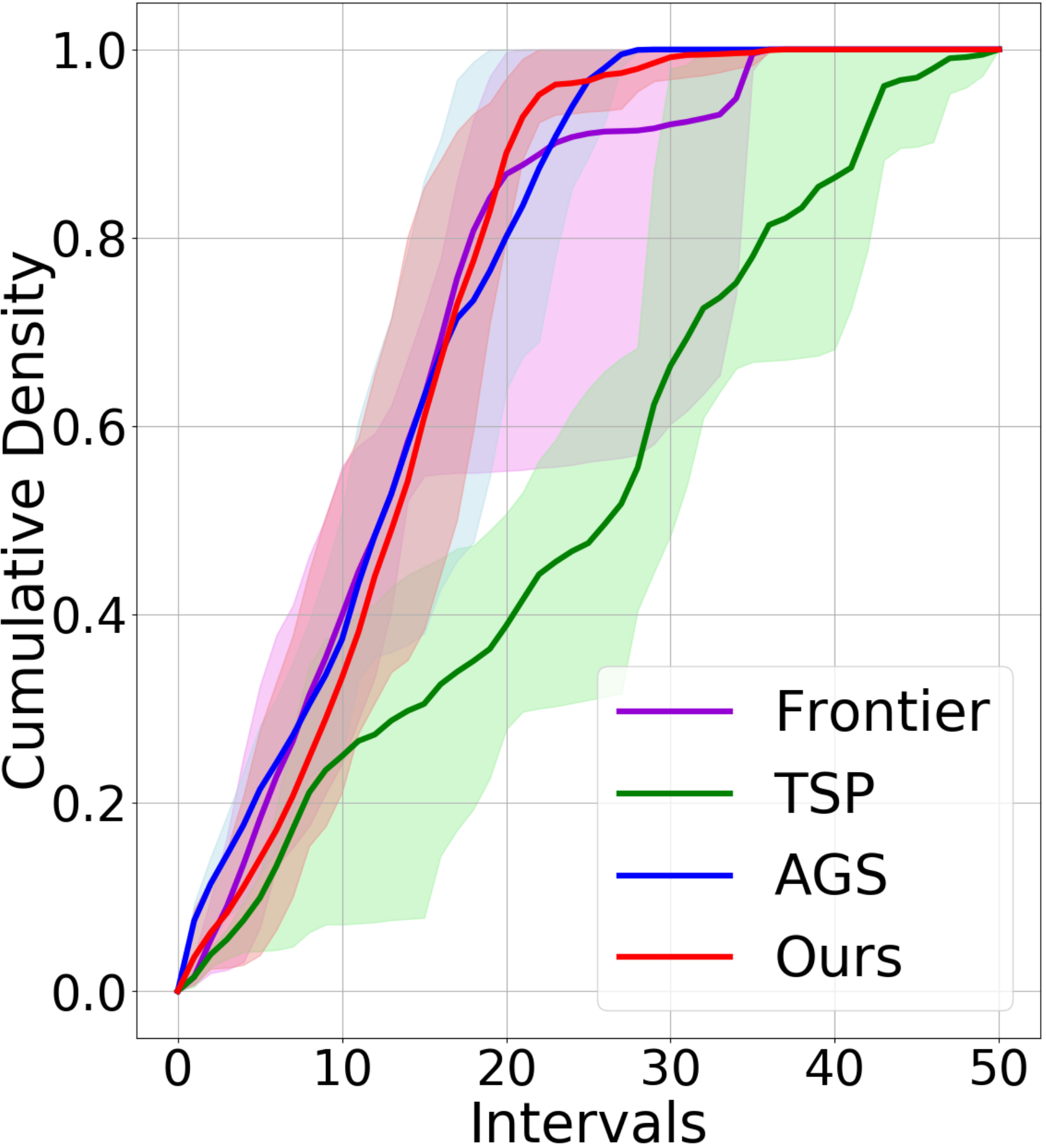}}
\caption{Cumulative density distribution of the robot poses uncertainty. Each curve (with the lower and upper bound) is the averaged result over five independent runs. The pose uncertainty is evaluated by the determinant of the covariance matrix associated with each pose (using GTSAM~\cite{gtsam}) in the pose graph, and the distribution of pose uncertainty is discretized into $50$ sub-intervals to compute its cumulative density function.}
\label{fig_cumulative}
\vspace{-15pt}
\end{figure}

\begin{figure}[t]
\vspace{1pt}
\centering
\subfloat[]{\includegraphics[width=.48\linewidth]{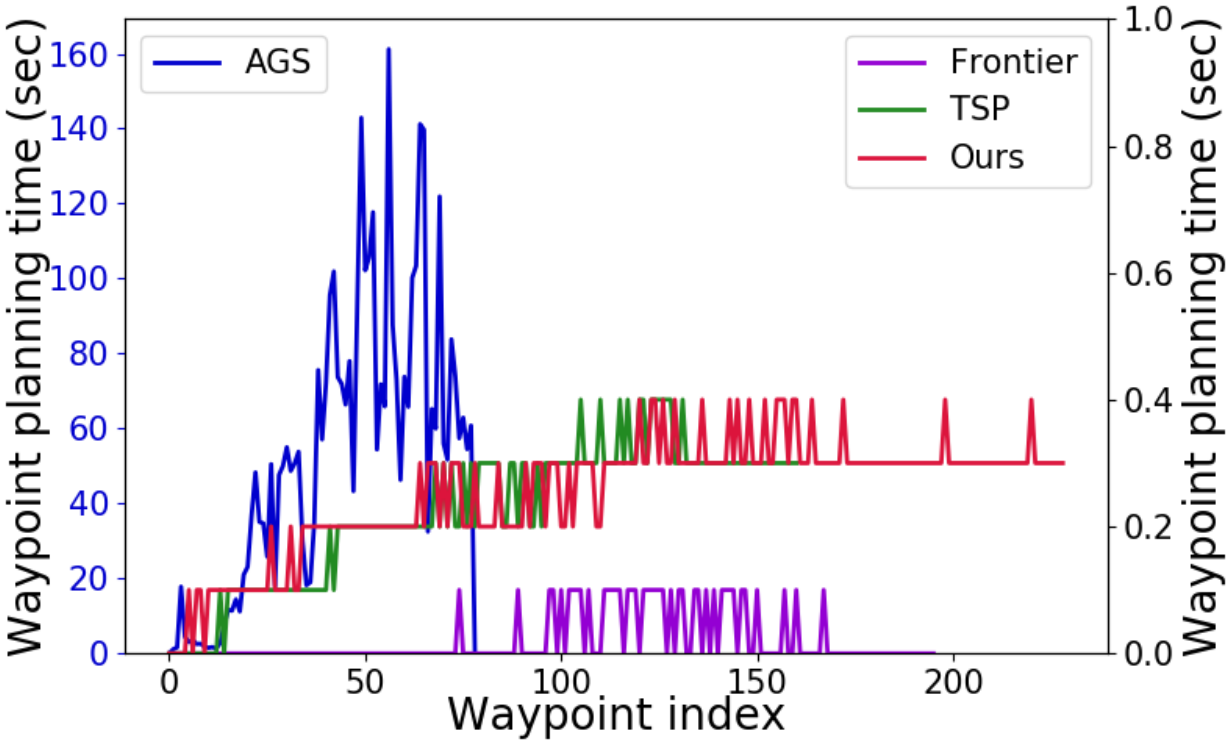}}\hspace{2pt}
\subfloat[]{\includegraphics[width=.49\linewidth]{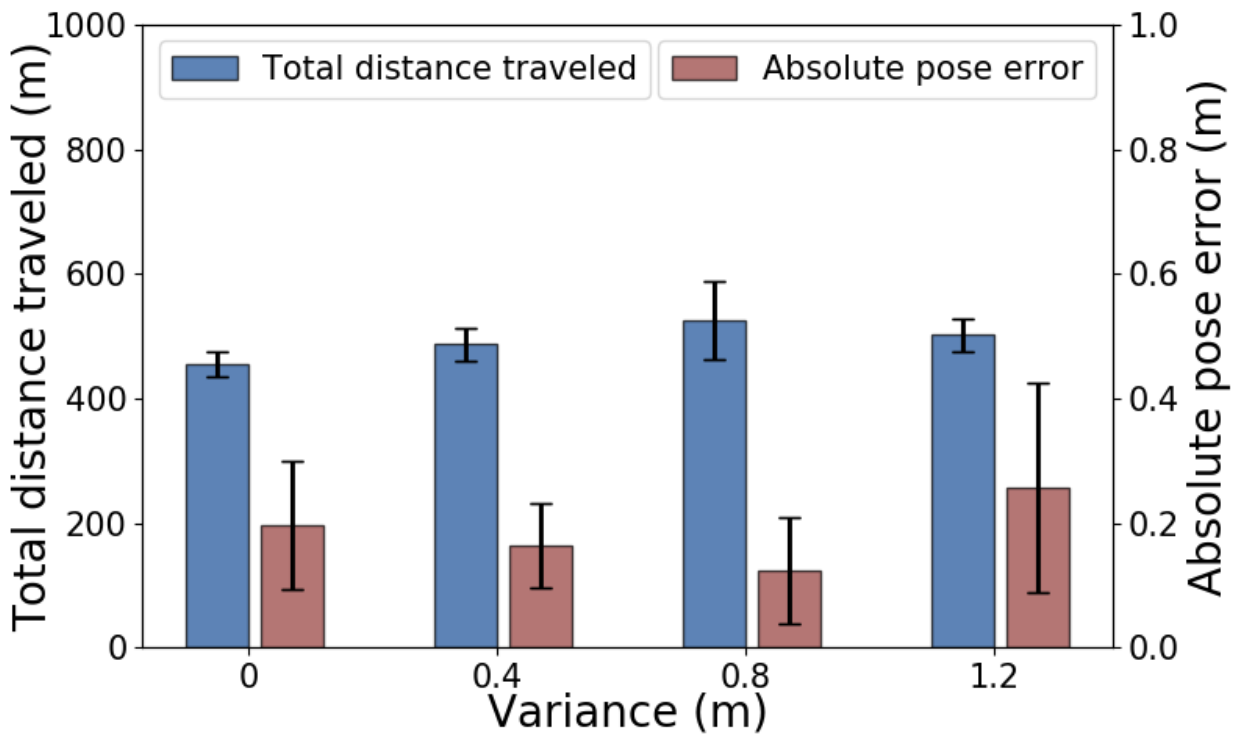}}
\caption{(a) A record of the waypoint planning time during the exploration of Env.~$1$. The left y-axis only corresponds to the AGS method, while the right y-axis corresponds to the other methods. (b) The total exploration distance and the RMSE of absolute pose error of the proposed method when Gaussian noise with variance $0.4m$, $0.8m$ and $1.2m$ are added to the positions of vertices in the prior graph of Env.~$1$. The averaged results and standard deviations over five independent runs are shown.}
\label{fig_plan_time}
\vspace{-15pt}
\end{figure}

Fig.~\ref{fig_cumulative} shows the cumulative density distributions of the pose uncertainty in the robot's pose graph.
The curve that converges faster to one indicates a higher proportion of poses have smaller estimation uncertainty.
Generally, the curves of our method converge faster to one than other methods, with obviously smaller variances among different runs.
In comparison, both the frontier-based and TSP-based methods show larger variations and converge slower.
The method~\cite{placed_fast_2021} shows a comparable uncertainty distribution to ours in Env.~$1$, $2$, and $4$, at the cost of longer distances that introduces extra odometry noise to the pose estimation as in Tab.~\ref{tab_exploration}.

Moreover, we compared the waypoint planning time during the exploration of Env.~$1$, as shown in Fig.~\ref{fig_plan_time}(a). 
The waypoints are low-level navigation goals decided after evaluating existing frontiers or loop-closing opportunities, and thus the planning time reflects the online decision-making efficiency of the compared methods.
As the active SLAM method~\cite{placed_fast_2021} needs to construct a local pose graph when evaluating each frontier, it usually takes tens of seconds to find the best one and thus is computationally inefficient.
Our method has much better runtime efficiency than~\cite{placed_fast_2021} because we only evaluate the pose graph reliability with its coarse structure. 
The runtime efficiency of our method is similar to the TSP-based method, but with additional capabilities of online prior graph updating and path replanning. 
Moreover, we evaluate the performance of the proposed method using the prior graph of Env.~$1$ with different levels of noise, as shown in Fig.~\ref{fig_plan_time}(b).
The averaged total exploration distance and the absolute pose error are less affected, verifying our exploration strategy is robust to the slight variations in the prior graph, 
as explained in Sec.~\ref{sec_hierarchical}.

\begin{figure}[t]
\centering
\subfloat[]{\includegraphics[width=.175\linewidth]{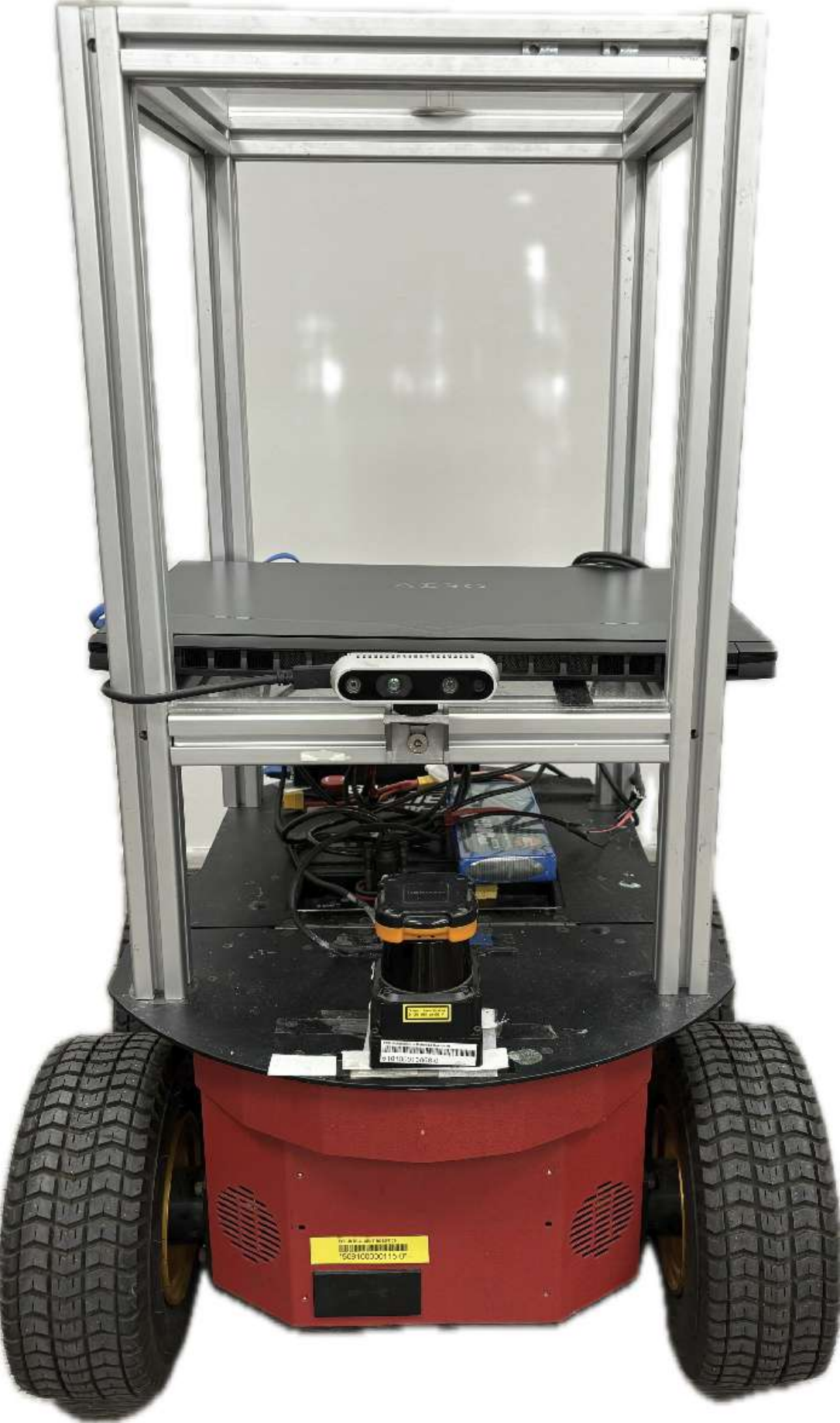}}\hspace{8pt}
\subfloat[]{\includegraphics[width=.75\linewidth]{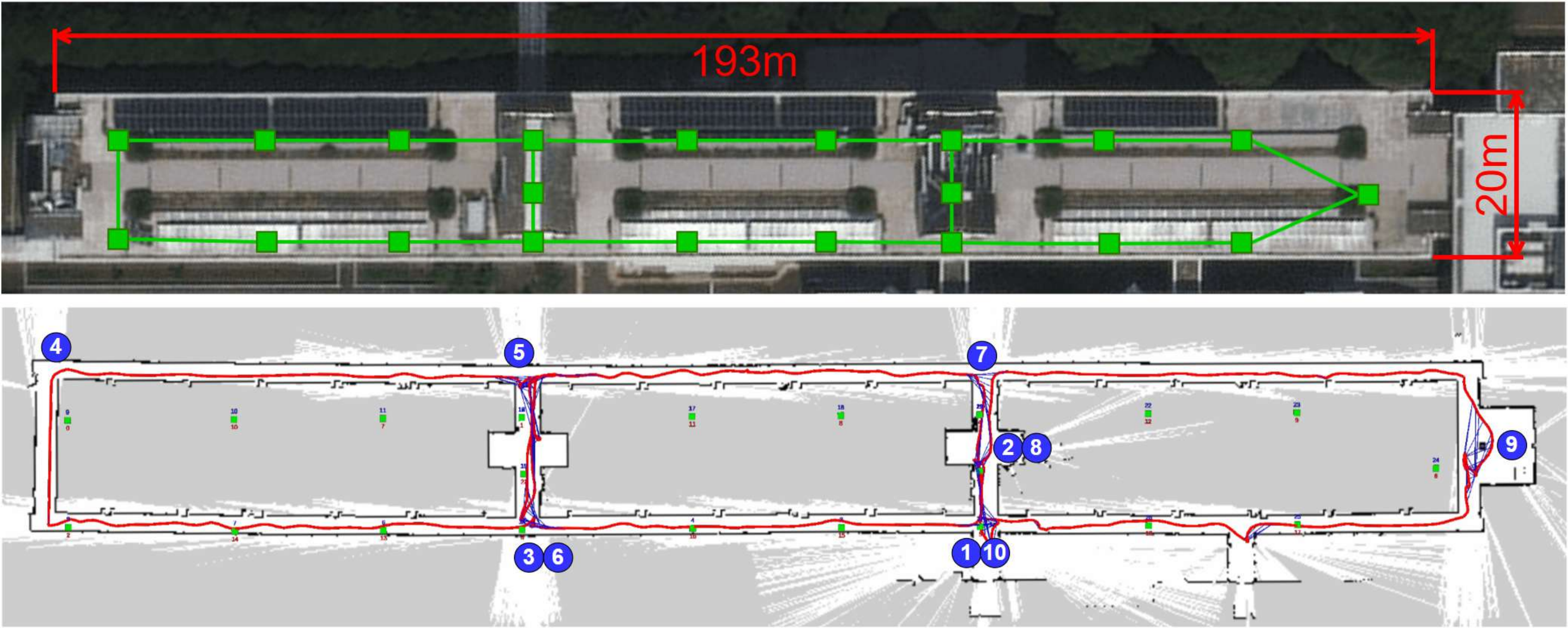}}
\caption{(a) The Pioneer3AT robot equipped with wheel encoders, a Hokuyo UTM-30LX laser, and Intel RealSense D435 camera (for visualization only); (b) Top: topview of the building from satellite image; Bottom: the SLAM-built map after exploration and the robot pose graph.}
\label{fig_real_world}
\vspace{-15pt}
\end{figure}

\subsection{Real-World Experiment with Pioneer3AT Robot}

We validated the proposed method by exploring a structured corridor environment on the NTU campus, as shown in Fig.~\ref{fig_real_world}.
The topo-metric information is obtained from the satellite map, and about half of the vertices in the prior graph are wrongly placed within obstacles.
The proposed method uses such an inaccurate prior graph, finds the path that quickly covers the environment and includes three active loop-closing actions to mitigate the odometry drift during exploration. 
The obtained map is shown in Fig.~\ref{fig_real_world} (b).
More details are included in the attached video.

\section{Conclusion}
This letter proposes a SLAM-Aware path planner for autonomous exploration based on a prior topo-metric graph of the environment, which finds a high-level path to guide the robot to quickly cover unknown regions while forming a reliable SLAM pose graph with informative loop closures.
A hierarchical exploration framework is designed to incorporate the proposed planner, with additional features including online prior graph update and path replanning.
Experimental results verify that the proposed method achieves a superior balance between mapping accuracy and exploration efficiency
in various environments.
The proposed method is specifically applicable for exploring environments with rich topology information, like industrial factories, architectural complexes, or underground mines, where the topo-metric information can be effectively exploited.
Future work is to construct the topo-metric graph online and extend it to the multi-robot case.


\bibliographystyle{ieeetr} 
\bibliography{main} 

\end{document}